\documentclass[jair,twoside,11pt,theapa]{article}
\usepackage{jair, theapa, rawfonts}

\jairheading{47}{2013}{649-695}{02/13}{08/13}
\ShortHeadings{Protecting Privacy thru Distributed Computation in Multi-agent Decision Making}
{L\'eaut\'e \& Faltings}
\firstpageno{649}
	
\newcommand{\citet}[1]{\citeauthor{#1}~\citeyear{#1}}

\usepackage {algorithm}
\usepackage[noend]{algorithmic}

\algsetup{indent=2em}

\usepackage {graphicx}
\usepackage {subfigure}
\usepackage{amsthm}
\usepackage{colortbl}
\usepackage {xfrac}
\usepackage {url}
\usepackage{amssymb}
\usepackage{tikz}

\let\oldmarginpar\marginpar
\renewcommand\marginpar[1]{\-\oldmarginpar[\raggedleft\footnotesize #1]%
{\raggedright\footnotesize #1}}

\newtheorem{theorem}{Theorem}

\newtheorem{definition}{Definition}

\begin{document}


\title{Protecting Privacy through Distributed Computation \\in Multi-agent Decision Making}

\author{\name Thomas L\'eaut\'e \email thomas.leaute@a3.epfl.ch \\
       \name Boi Faltings \email boi.faltings@epfl.ch \\
       \addr 
       Ecole Polytechnique F\'ed\'erale de Lausanne (EPFL) \\
       Artificial Intelligence Laboratory (LIA)\\
       Station 14 \\
       CH-1015 Lausanne, Switzerland}
\maketitle

\begin {abstract}
As large-scale theft of data from corporate servers is becoming increasingly
common, it becomes interesting to examine alternatives to the paradigm of centralizing
sensitive data into large databases. Instead, one could use cryptography and distributed
computation so that sensitive data can be supplied and processed in encrypted form, and
only the final result is made known. In this paper, we examine how such a paradigm can be used to implement \emph{constraint satisfaction}, a
technique that can solve a broad class of AI problems such as resource allocation, planning,
scheduling, and diagnosis. Most previous work on privacy in constraint satisfaction only attempted to protect specific types of information, in particular the feasibility of particular combinations of decisions. We formalize and extend these restricted notions of privacy by introducing four types of private information, including the feasibility of decisions and the final decisions made, but also the identities of the participants and the topology of the problem. We present distributed algorithms that allow
computing solutions to constraint satisfaction problems while maintaining these four types of
privacy. We formally prove the privacy properties of these algorithms, and show experiments
that compare their respective performance on benchmark problems. 
\end {abstract}

\section {Introduction} 

Protecting the privacy of information is becoming a crucial concern to many users of the increasingly ubiquitous Information and Communication Technologies. Companies invest a lot of effort into keeping secret their internal costs and their future development strategies from other actors on the market, most importantly from their competitors. Individuals also have a need for privacy of their personal information: for instance, carelessly disclosing one's activity schedule or location might reveal to burglars opportunities to break into one's home. On the other hand, accessing and using such private information is often necessary to solve problems that depend on these data. In the context of supply chain management, companies need to exchange information with their contractors and subcontractors about the quantities of goods that must be produced, and at what price. When scheduling meetings or various events with friends or co-workers, individuals are confronted with the challenge of taking coordinated scheduling decisions, while protecting their respective availability schedules. 

Artificial Intelligence can be a crucial tool to help people make better decisions under privacy concerns, by delegating part or all of the decision problem to personal \emph {intelligent agents} executing carefully chosen algorithms that are far too complex to be performed by the human alone. In particular, the framework of \emph {Constraint Satisfaction Problems (CSPs)} is a core AI technology that has been successfully applied to many decision-making problems, from configuration to scheduling, to solving strategic games. Here we show how distributed AI algorithms can be used to solve such CSPs, while providing strong guarantees on the privacy of the problem knowledge, through the use of techniques borrowed from cryptography. This makes it possible to solve coordination problems that depend on secret data, without having to reveal these data to other parties. On the other hand, distributed, encrypted computation involving message exchange has a cost in terms of performance, such that a suitable tradeoff between privacy and scalability must be found.

\subsection {Motivating Examples}

In this paper, we present a set of novel, privacy-protecting algorithms for \emph {Distributed Constraint Satisfaction Problems (DisCSPs)}, a wide class of multi-agent decision-making problems with applications to many problems such as configuration, scheduling, planning, design and diagnosis. We consider three examples to illustrate the privacy requirements that might arise: \emph {meeting scheduling}, \emph {airport slot allocation}, and \emph {computing game equilibria}. 

In a meeting scheduling problem \cite{Maheswaran04}, a number of meetings need to be scheduled, involving possibly overlapping sets of participants. Taking into account their respective availability constraints, all participants to any given meeting must agree on a time for the meeting. One given participant can be involved in multiple meetings, which creates constraints between meetings. In this problem class, participants usually want to protect the privacy of their respective availability schedules, as well as the lists of meetings they are involved in. 

Another problem class is airport slot allocation \cite {Rassenti82}, where airlines express interests in combinations of takeoff and landing time slots at airports, corresponding to possible travel routes for their aircraft. While the end goal for the airports is to efficiently allocate their slots to airlines, from the point of view of the airlines it is crucial that the combinations of slots they are interested in remain private, because they indicate the routes they intend to fly, which is sensitive strategic information that they want to hide from their competitors. 

Finally, consider the general class of one-shot strategic games, such as the \emph {party game} \cite {Singh04}: the players are invited to a party, and must decide whether to attend, based on their respective intrinsic costs of attendance, and on whether the people they like or dislike also choose to attend. Players would best play strategies that form a \emph {Nash equilibrium}, where no single player can be better off by deviating from its chosen strategy. The problem of computing such an equilibrium is a typical example of a multi-agent decision-making problem, in which privacy is an issue: players do not necessarily want to reveal their attendance costs, nor whether they like or dislike another invitee.

\subsection {Four Types of Private Information}

As can be seen in the previous examples, the information that participants would like to keep private can differ in nature; we propose to classify it into four privacy types. We only briefly introduce and illustrate them here; more formal definitions are given in Section~\ref {sec:privacy_defs}. 

\begin{enumerate}
\item \emph{Agent privacy} relates to the identities of the participants. Consider for instance a CEO who wants to schedule two meetings respectively with a journalist and with another company's CEO. Revealing to the journalist the other CEO's involvement in the decision-making problem could leak out the companies' plans to merge. In this case agent privacy can be considered critical. 

\item \emph{Topology privacy} covers information about the \emph {presence} of constraints. This is the type of critical information that airline companies want to keep secret in the airport slot allocation problem: the presence of a constraint between an airline and a specific airport reveals the airline's strategic plans to offer flights to and from this airport. 

\item \emph{Constraint privacy} is about the \emph {nature} of the constraints. This covers for instance the participants' availability schedules in the meeting scheduling problem, and, in the party game, whether a player likes or dislikes other invitees. 

\item \emph{Decision privacy} has to do with the solution that is eventually chosen to the problem. Depending on the problem class, this type of privacy may or may not be relevant. In the meeting scheduling problem, the time chosen for each meeting necessarily has to be revealed to all participants of the meeting; however it can be desirable to hide this information from non-attendees. 
\end{enumerate}

Like in previous work on privacy in DisCSP, we assume that the participants are \emph {honest, but curious} \cite {Goldreich04}, in that they honestly follow the algorithm, but are interested in learning as much as possible from other agents' private information based on the messages exchanged. Note that this honesty assumption does not mean that all agents are assumed to faithfully report their true constraints to the algorithm; they may be tempted to strategize by reporting slightly different constraints, hoping that this would lead the algorithm to select a solution to the problem that they deem preferable to them. This issue of \emph {incentive-compatibility} has been addressed in related work such as by \citet {Petcu08}, and is orthogonal to the issue of privacy addressed in this paper. Furthermore, an agent would take a risk in reporting constraints different from its true constraints: reporting relaxed constraints could yield a solution that violates its true constraints and would therefore not be viable, while reporting tighter constraints could make the overall problem infeasible and the algorithm fail to find any solution at all. 

On the other hand, our algorithms depart from previous work in two respects. First, previous work almost exclusively focused on constraint privacy, most often ignoring agent, topology and decision privacy. We show how to address all four types, and the algorithms we propose correspond to various points in the tradeoff between different levels of privacy and efficiency. Second, while most of the literature focuses on \emph{quantitatively measuring and reducing} the amount of privacy loss in various DisCSP algorithms, we have developed algorithms that give \emph {strong guarantees} that certain pieces of private information will not be leaked. In contrast, in previous privacy-protecting algorithms, it is typically the case that any piece of private information may be leaked with some (small) probability. 

The rest of this paper is organized as follows. Section~\ref {sec:preliminaries} first formally defines the DisCSP framework and the four aforementioned types of privacy. Section~\ref {sec:P_DPOP} then presents a first algorithm, called \emph {P-DPOP$^+$}. Section~\ref {sec:P_DPOP_value} then describes the \emph{P$^{\sfrac{3}{2}}$-DPOP$^+$} algorithm, which is a variant that achieves a higher level of decision privacy, at the expense of an additional computational overhead. Another variant, called \emph {P$^2$-DPOP$^+$}, is introduced in Section~\ref {sec:ElGamal_UTILpropagation} in order to further improve constraint privacy. Finally, Section~\ref {sec:results} compares the performance of these algorithms with the previous state of the art, on several classes of benchmarks.

\section {Preliminaries} 
\label {sec:preliminaries}

This section first formally defines the DisCSP framework (Section~\ref {sec:DisCSP}), and then introduces four types of privacy (Section~\ref {sec:prelim:privacy}).

\subsection {Distributed Constraint Satisfaction}
\label {sec:DisCSP}

After providing a formal definition of Distributed Constraint Satisfaction (Section~\ref {sec:DisCSPdef}), we recall some existing algorithms for DisCSP and its optimization variant (Section~\ref {sec:algos}).

\subsubsection {Definition}
\label {sec:DisCSPdef}

A Distributed Constraint Satisfaction Problem can be formally defined as follows. 

\begin{definition}[DisCSP]
A discrete {\em DisCSP} is a tuple $< \mathcal {A}, \mathcal {X}, a, \mathcal {D}, \mathcal {C} >$:
\begin{itemize}
	\item ${\cal A} = \{a_1,...,a_k\}$ is a set of \emph{agents}; 
	\item ${\cal X} = \{x_1,...,x_n\}$ is a set of \emph{variables}; 
	\item $a: \mathcal {X} \rightarrow \mathcal {A}$ is a mapping that assigns the control of each variable $x_i$ to an agent~$a(x_i)$; 
	\item ${\cal D} = \{D_1,...,D_n\}$ is a set of finite variable \emph{domains}; variable~$x_i$ takes values in~$D_i$; 
	\item ${\cal C} = \{c_1,...,c_m\}$ is a set of \emph{constraints}, where each $c_i$ is a $s(c_i)$-ary function of scope $(x_{i_1}, \cdots,x_{i_{s(c_i)}})$, $c_i : D_{i_1} \times .. \times D_{i_{s(c_i)}} \rightarrow \{\mathtt{false}, \mathtt{true}\}$, assigning $\mathtt{false}$ to infeasible tuples, and $\mathtt{true}$ to feasible ones. 
\label{def:DCOP}
\end{itemize}
A \emph{solution} is a complete assignment such that the conjunction $\bigwedge_{c_i \in {\cal C}}c_i = \mathtt {true}$, which is the case exactly when the assignment is consistent with all constraints. 
\end{definition}

Some of  the important assumptions of the DisCSP framework are the following. First, we assume that all the details of a given constraint~$c_i$ are known to all agents involved; if an agent wants to keep some constraints private, it should formulate them in such a way that they only involve variables it controls. Furthermore, we assume that two neighboring agents (i.e. agents that share at least one constraint) are able to communicate with each other securely, and that messages are delivered in FIFO order and in finite time. On the other hand, we assume that two non-neighboring agents initially ignore everything about each other, even including their involvement in the problem. In particular, a DisCSP algorithm that protects \emph {agent privacy} should not require them to communicate directly, nor should it even allow them to discover each other's presence. Finally, we assume each agent honestly follows the protocol, and we focus on preventing private information leaks to other agents. 

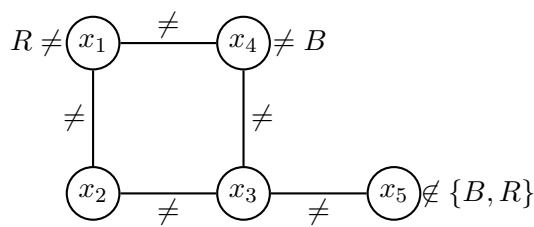
\begin{figure}
\begin{center}
\begin {tikzpicture}

\tikzstyle{var}=[circle, draw = black, thick, fill = white, minimum size = 0.7cm, inner sep=0pt, draw]

\node[var] (x2) at (0, 0) {$x_2$};
\node[var] (x3) at (2, 0) {$x_3$};
\draw (x2) -- (x3) [thick];
\node at (1, -.25) {$\neq$};

\node[var] (x5) at (4, 0) {$x_5$};
\draw (x3) -- (x5) [thick];
\node at (3, -.25) {$\neq$};
\node at (5.12, 0) {$\not\in \{B, R\}$};

\node[var] (x1) at (0, 2) {$x_1$};
\draw (x2) -- (x1) [thick];
\node at (-.25, 1) {$\neq$};
\node at (-.75, 2) {$R \neq$};

\node[var] (x4) at (2, 2) {$x_4$};
\draw (x1) -- (x4) [thick];
\node at (2.75, 2) {$\neq B$};
\node at (1, 2.25) {$\neq$};

\draw (x4) -- (x3) [thick];
\node at (2.25, 1) {$\neq$};

\end {tikzpicture}

\caption{The DisCSP constraint graph for a simple graph coloring problem instance. }
\label{fig:constraint_graph}
\end{center}
\end{figure}

Figure~\ref {fig:constraint_graph} introduces a simple graph coloring problem instance that will be used to illustrate the algorithms throughout the rest of this paper. We assume that the five nodes in the graph correspond to five different agents, which must each choose a color among red, blue and green. These decisions are modeled by the five variables $x_1, \ldots, x_5$ with domains $\{ R, B, G \}$. Each agent may express a secret, unary constraint on its variable; for instance, $x_1$~does not want to be assigned the color red. Binary, inequality constraints are imposed between each pair of neighboring nodes, and are only known to the two agents involved. 

\emph {Distributed Constraint Optimization (DCOP)} is an extension of the DisCSP formalism, in which constraints specify not only which variable assignments are feasible or infeasible, but also assign \emph {costs} (or \emph {utilities}) to these assignments. An (optimal) solution to such a DCOP is then one that minimizes the sum of all costs (or maximizes the sum of all utilities). The algorithms in this paper can easily be generalized to solve DCOPs, with a complexity increase that is at most linear in an upper bound on the (assumed integer) cost of the optimal solution. Such a generalization is left outside the scope of this paper for the sake of conciseness, and has been addressed by \citet {Leaute11} and \citet {Leaute11b}.

\subsubsection {Complete Algorithms for DisCSPs}
\label {sec:algos}

A  range of distributed algorithms exist in the literature to solve DisCSPs and DCOPs. They can be seen as belonging to two classes, depending on how they order variables. The largest class consists of algorithms that order the variables along a \emph {linear order}, such as \emph{ABT}~\cite{Yokoo92}, \emph{AWC}~\cite{Yokoo95}, SynchBB~\cite{Hirayama97}, \emph{AAS}~\cite{Silaghi00}, AFC \cite {Meisels03}, DisFC~\cite{Brito03}, \emph{(Comp)APO} \cite{Mailler03,Grinshpoun08}, ConcDB \cite {Zivan04}, AFB~\cite{Gershman06} and ConcFB \cite {Netzer10}. The linear order may be chosen and fixed initially before the algorithm is run, or dynamically revised online. 

In the second class, variables are ordered along a tree-based \emph{partial order}. This includes \emph{ADOPT}~\cite{Modi05a} and its variants such as BnB-ADOPT \cite{Yeoh10} and BnB-ADOPT$^+$ \cite {Gutierrez10}, \emph{DPOP}~\cite{Petcu05} and its countless variants, and NCBB~\cite{Chechetka06}, which all order the variables following a \emph {pseudo-tree} (Definition~\ref {def:DFS}). Among the aforementioned pseudo-tree-based algorithms, DPOP is the only one using \emph {Dynamic Programming (DP)}, while all others are based on \emph {search}. Other algorithms have been proposed that perform DP on different partial variable orders: Action-GDL uses \emph {junction trees} \cite {Vinyals10}, and DCTE \emph {cluster trees} \cite {Brito10a}.

\subsubsection {The DPOP Algorithm}
\label {sec:DPOP}

The DPOP algorithm was originally designed to solve optimization problems (DCOPs) and described in terms of utility maximization. One way to apply it to pure satisfaction problems (DisCSPs) is to first reformulate the DisCSP into a \emph {Max-DisCSP}, in which the constraints are no longer boolean but rather take values in $\{ 0, 1 \}$, where $0$~stands for feasibility and $1$~for infeasibility. The cost-minimizing variant of DPOP (described below) can then be applied to find a solution with minimal cost, where the cost (hereafter called \emph {feasibility value}) corresponds to the number of constraint violations (which we want to be equal to~$0$). 

\begin {algorithm}[b!]
\begin {algorithmic}[1]
\REQUIRE a pseudo-tree ordering of the variables; $p_x$ denotes $x$'s parent

\STATE \COMMENT {(\emph {UTIL propagation}) Propagate feasibility values up the pseudo-tree:} 	\label {algo:DPOP:UTILpropagation}

\STATE $m(x, p_x, \cdot) \leftarrow \Sigma_{c \in \left\{ c' \in \mathcal {C} ~|~ x \in scope(c') ~\wedge~ scope(c') \cap \left( children_x \cup pseudo\_children_x \right) = \emptyset \right\}} c (x, \cdot)$ 	\label {algo:DPOP:UTILpropagation:local_join}

\vspace{5pt}
\STATE \COMMENT {Join with received messages:}
\FOR {each $y_i \in children_x$}
	\STATE Wait for the message (FEAS, $m_i(x, \cdot)$) from $y_i$ 		\label {algo:DPOP:UTILpropagation:get_message}
	\STATE $sep_{y_i} \leftarrow scope(m_i)$ 	\label {algo:DPOP:UTILpropagation:separator}
	\STATE $m( x, p_x, \cdot ) \leftarrow m(x, p_x, \cdot) + m_i(x, \cdot )$ 	\label {algo:DPOP:UTILpropagation:join} 
\ENDFOR

\vspace{5pt}
\STATE \COMMENT {Project out $x$:}
\IF{$x$ is not the root variable}

	\STATE $x^*(p_x, \cdot) \leftarrow \arg\min_x \left\{ m( x, p_x, \cdot ) \right\}$ \label {algo:DPOP:UTILpropagation:argmin}
	
	\STATE Send the message (FEAS, $m( x^*(p_x, \cdot), p_x, \cdot )$) to $p_x$ 	\label {algo:DPOP:UTILpropagation:sendUTIL}
\ENDIF
\STATE \textbf {else} $x^* \leftarrow \arg\min_x \left\{ m( x ) \right\}$ \label {algo:DPOP:UTILpropagation:root_value} \COMMENT {$m( x, p_x, \cdot )$ actually only depends on $x$}

\vspace {10pt}
\STATE \COMMENT {(\emph {VALUE propagation}) Propagate decisions top-down along the pseudo-tree:} 	\label {algo:DPOP:VALUEpropagation_start}
\IF {$x$ is not the root}
	\STATE Wait for message (DECISION, $p_x^*, \cdot$) from parent $p_x$
	\STATE $x^* \leftarrow x^*(p_x = p_x^*, \cdot)$ 	\label {algo:DPOP:lookup}
\ENDIF
\STATE \textbf {for} each $y_i \in children_x$ \textbf {do} send message (DECISION, $sep_{y_i}^*$) to $y_i$	\label {algo:DPOP:VALUEpropagation_end}
\end {algorithmic}
\caption {Overal DPOP algorithm, for variable~$x$}
\label {algo:DPOP}
\end {algorithm}

\paragraph {Overview of the Algorithm}

DPOP is an instance of the general bucket elimination scheme by \citet{Dechter03}, performed distributedly (Algorithm~\ref {algo:DPOP}). It requires first arranging the constraint graph into a \emph {pseudo-tree}, formally defined as follows. 

\begin {definition}[Pseudo-tree]
\label {def:DFS}
A \emph {pseudo-tree} is a generalization of a tree, in which a node is allowed to have links \emph {(back-edges)} with remote ancestors \emph {(pseudo-parents)} and with remote descendants \emph {(pseudo-children)}, but never with nodes in other branches of the tree. 
\end {definition}

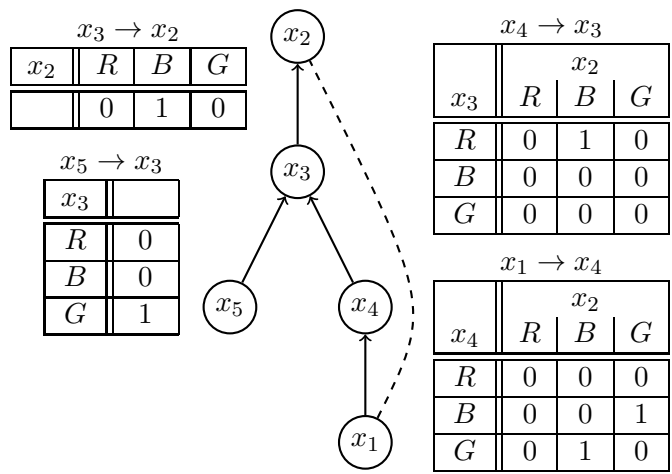
\begin{figure}[tp]
\begin{center}
\begin {tikzpicture} [scale=.9]

\tikzstyle{var}=[circle, draw = black, thick, fill = white, minimum size = 0.7cm, inner sep=0pt, draw]

\node[var] (x1) at (1, 0) {$x_1$};
\node[var] (x4) at (1, 2) {$x_4$};
\draw[->] (x1) -- (x4) [thick];

\node[var] (x5) at (-1, 2) {$x_5$};
\node[var] (x3) at (0, 4) {$x_3$};
\draw[->] (x5) -- (x3) [thick];
\draw[->] (x4) -- (x3) [thick];

\node[var] (x2) at (0, 6) {$x_2$};
\draw[->] (x3) -- (x2) [thick];

\draw (x1) .. controls (2, 2) .. (x2) [thick,dashed];

\node at (-2.75, 3) {
\begin{tabular}{|c||c|}
\multicolumn {2}{c}{$x_5 \rightarrow x_3$}	\\
\hline
$x_3$	&	\\
\hline \hline
$R$		&	$0$	\\
\hline
$B$		&	$0$		\\
\hline
$G$		&	$1$		\\
\hline
\end{tabular}
};

\node at (3.75, 1.25) {
\begin{tabular}{|c||c|c|c|}
\multicolumn {4}{c}{$x_1 \rightarrow x_4$}	\\
\hline
		&	\multicolumn{3}{c|}{$x_2$}	\\
$x_4$	&	$R$		&	$B$		& 	$G$ 	\\
\hline \hline
$R$		&	$0$		&	$0$		&	$0$	\\
\hline
$B$		&	$0$		&	$0$		&	$1$	\\
\hline
$G$		&	$0$		&	$1$	&	$0$	\\
\hline
\end{tabular}
};

\node at (3.75, 4.75) {
\begin{tabular}{|c||c|c|c|}
\multicolumn {4}{c}{$x_4 \rightarrow x_3$}	\\
\hline
		&	\multicolumn{3}{c|}{$x_2$}	\\
$x_3$	&	$R$				&	$B$				& 	$G$ 	\\
\hline \hline
$R$		&	$0$		&	$1$	&	$0$	\\
\hline
$B$		&	$0$		&	$0$		&	$0$	\\
\hline
$G$		&	$0$		&	$0$		&	$0$	\\
\hline
\end{tabular}
};

\node at (-2.5, 5.5) {
\begin{tabular}{|c||c|c|c|}
\multicolumn {4}{c}{$x_3 \rightarrow x_2$}	\\
\hline
$x_2$	&	$R$		&	$B$	& 	$G$ 	\\
\hline \hline
		&	$0$		&	$1$	&	$0$	\\
\hline
\end{tabular}
};

\end {tikzpicture}

\caption{Multiparty dynamic programming computation of a feasible value for $x_2$, based on a pseudo-tree arrangement of the constraint graph in Figure~\ref {fig:constraint_graph}. The dashed edge represents a back-edge with a pseudo-parent in the pseudo-tree. }
\label{fig:dfs}
\end{center}
\end{figure}

A  pseudo-tree arrangement of the constraint graph in Figure~\ref {fig:constraint_graph} is illustrated in Figure~\ref {fig:dfs}. This pseudo-tree naturally decomposes the original problem into two, loosely coupled subproblems, corresponding to the two branches, which will perform the rest of the algorithm in parallel. Figure~\ref {fig:dfs} also shows the \emph {FEAS} messages (originally called \emph {UTIL} messages in the context of utility maximization) that are exchanged during the propagation of feasibility values, following a multi-party dynamic programming computation (lines \ref {algo:DPOP:UTILpropagation} to~\ref {algo:DPOP:UTILpropagation:root_value}). In this part of the algorithm, all messages travel bottom-up along tree edges. Consider for instance the message sent by agent~$a(x_5)$ to its parent agent~$a(x_3)$. This message is the result of the \emph {projection} (lines \ref {algo:DPOP:UTILpropagation:argmin} and~\ref {algo:DPOP:UTILpropagation:sendUTIL}) of variable~$x_5$ out of the conjunction (line~\ref {algo:DPOP:UTILpropagation:local_join}) of $x_5$'s two constraints $x_5 \neq x_3$ and $x_5 \neq B$, and summarizes the minimal number of constraint violations that $a(x_5)$ can achieve, as a function of the ancestor variable~$x_3$. More generally, each message sent by a variable~$x$ summarizes the minimal number of constraint violations achievable for the aggregate subproblem owned by the entire subtree rooted at~$x$, as a function whose scope is called the \emph {separator} of~$x$ (line~\ref {algo:DPOP:UTILpropagation:separator}). In DPOP, the separator of~$x$ necessarily includes $x$'s parent~$p_x$, and potentially other ancestor variables; this is indicated by the notation $m(p_x, \cdot)$. For instance, the message $x_4 \rightarrow x_3$ summarizes the minimal number of constraint violations achievable for the entire subtree rooted at~$x_4$, as a function of $x_4$'s separator $\{ p_{x_4} = x_3, x_2 \}$. Notice that the separator of a variable~$x$ can contain variables that are not neighbors of~$x$; for example, $x_2$ is in $x_4$'s separator because a descendent of~$x_4$ has a constraint with~$x_2$. In the privacy-aware algorithms presented later in this paper, this notion of separator is extended to allow for separators that do not necessarily include the parent variable, and that may include multiple \emph {codenames} referring to the same variables, which might not necessarily be ancestors in the pseudo-tree. 

Upon receiving the messages $x_5 \rightarrow x_3$ and $x_4 \rightarrow x_3$ (line~\ref {algo:DPOP:UTILpropagation:get_message}), agent~$a(x_3)$ \emph {joins} them (line~\ref {algo:DPOP:UTILpropagation:join}) with its constraint $x_3 \neq x_2$. Variable~$x_3$ is then projected out of the resulting joint table, which produces the message $x_3 \rightarrow x_2$ (lines \ref {algo:DPOP:UTILpropagation:argmin} and~\ref {algo:DPOP:UTILpropagation:sendUTIL}). At the end of this feasibility propagation (line~\ref {algo:DPOP:UTILpropagation:root_value}), the root variable~$x_2$ chooses a value~$x_2^*$ for itself that minimizes the number of constraint violations over the entire problem (e.g. $x_2^* = R$). This decision can then be propagated downwards along tree-edges via \emph{DECISION} messages (originally called \emph {VALUE} messages) until all variables have been assigned optimal values (lines \ref {algo:DPOP:VALUEpropagation_start} to~\ref {algo:DPOP:VALUEpropagation_end}).

\paragraph {Complexity}

Given a pseudo-tree ordering of the $n$~variables, DPOP's bottom-up and top-down phases each exchange exactly $(n-1)$ messages (one through each tree edge). However, while each DECISION message contains at most $(n-1)$ variable assignments, the FEAS message sent by a given variable~$x$ can contain exponentially many feasibility values, because it contains a table representation of a function~$m$ of $| sep_x |$ variables. The size of the largest FEAS message is therefore $O(D_{\max}^{sep_{\max}})$, where $D_{\max}$ is the size of the largest variable domain, and $sep_{\max} = \max_x | sep_x | < n - 1$ is the \emph {width} of the pseudo-tree. In the best case, the width is equal to the \emph {treewidth} of the constraint graph; however finding a pseudo-tree that achieves this minimal width is NP-hard. In practice, the pseudo-tree is generated by a heuristic, distributed, depth-first traversal of the constraint graph (Online Appendix~1), producing a so-called \emph {DFS tree} that is a pseudo-tree in which all parent-child relationships are between neighbors in the constraint graph. Since DPOP exchanges $(n-1)$ FEAS messages, its overall complexity in terms of runtime (measured in number of constraint checks), memory, and information exchange is~$O(n \cdot D_{\max}^{sep_{\max}})$.

\paragraph {Privacy Properties}

The privacy-aware algorithms in Section~\ref {sec:P_DPOP} are based on DPOP, because of two desirable properties that allow for higher levels of privacy. First, DPOP only requires message exchanges between neighboring agents, provided that the pseudo-tree used is a DFS tree; this is necessary to protect \emph {agent privacy}. \citet {Greenstadt06} made the opposite claim that pseudo-trees are detrimental to privacy compared to linear orderings; however this claim is only valid if the only type of privacy considered is \emph {constraint privacy}, and does not hold if \emph {agent privacy} and \emph {topology privacy} are guaranteed, i.e. if the pseudo-tree is not publicly known to all agents. The second, DP-inherited property is that DPOP's performance does not depend on \emph {constraint tightness}, i.e. how easy or hard it is to satisfy each constraint. For all other, search-based algorithms, inferences on the constraint tightness can be made by observing the runtime or the amount of information exchanged~\cite{Silaghi04}. In the case of meeting scheduling problems, constraint tightness maps directly to the participants' levels of availability, which is private information. In application domains where this leak of constraint tightness is tolerable, algorithms based on search rather than DP can be used, and many of the privacy-enhancing techniques presented in this paper for DPOP are also applicable to search-based algorithms.

\subsection {Privacy in DisCSPs}
\label {sec:prelim:privacy}

Section~\ref {sec:privacy_defs} formally defines the four types of privacy considered in this paper. Section~\ref {sec:previous_privacy} then recalls previous work that attempted to address various subsets of these privacy types.

\subsubsection {Privacy Definitions}
\label {sec:privacy_defs}

Definition~\ref {def:semi_private} introduces the concept of \emph {semi-private information} \cite {Faltings08a}, which may \emph {inevitably} be leaked by \emph {any} DisCSP algorithm. 

\begin {definition} [Semi-private information]
\label {def:semi_private}
\emph {Semi-private information} refers to information about the problem and/or its solution that an agent might consider private, but that can \emph {inevitably} be leaked to other agents by their views of the chosen solution to the DisCSP. 
\end {definition}

In other words, semi-private information covers everything a given agent can discover about other agents by making inferences simply based on its initial knowledge of the problem and on the values its variables take in the solution. For instance, in a  graph coloring problem involving only two colors, each node can infer the color of each of its neighbors from the color it was assigned in the chosen solution, provided that the solution is correct. Excluding semi-private information, we now distinguish four types of private information that agents may desire to protect~\cite {Faltings08a}. 

\begin {definition} [Agent privacy]
\label {def:agent_privacy}
No agent should be able to discover the identity, or even the existence of non-neighboring agents. A particular consequence of this type of privacy is that two agents should only be allowed to communicate directly if they share a constraint. 
\end {definition}

In Figure~\ref {fig:constraint_graph}, this means for instance that agent~$a(x_1)$ should not be able to discover the existence and identities of agents $a(x_3)$ and~$a(x_5)$. Even if no two non-neighboring agents communicate directly, agent privacy might still be leaked by the contents of messages; in this paper we propose a method based on \emph {codenames} to fully protect agent privacy. 

\begin {definition} [Topology privacy]
\label {def:topology}
No agent should be able to discover the existence of topological constructs in the constraint graph, such as nodes (i.e. variables), edges (i.e. constraints), or cycles, unless it owns a variable involved in the construct. 
\end {definition}

In Figure~\ref {fig:constraint_graph}, topology privacy means for instance that agent~$a(x_1)$ should not discover how many other neighbors $x_2$ has besides itself. However, $a(x_1)$ might discover the existence of a cycle involving $x_1$, $x_2$ and $x_4$. This is tolerated because $x_1$ is involved in this cycle, but $a(x_1)$ should not discover the length of the cycle (i.e. that $x_2$ and $x_4$ share a neighbor). 

\begin {definition} [Constraint privacy]
\label {def:constraint}
No agent should be able to discover the nature of a constraint that does not involve a variable it owns. 
\end {definition}

In Figure~\ref {fig:constraint_graph}, an example of a breach in constraint privacy would be if agent~$a(x_1)$ were able to discover that agent~$a(x_4)$ does not want to be assigned the color \emph {blue}. This is the type of privacy that the DisCSP literature mostly focuses on. 

\begin {definition} [Decision privacy]
\label {def:decision}
No agent should be able to discover the value that another agent's variable takes in the chosen solution (modulo semi-private information). 
\end {definition}

In a distributed graph coloring problem, this means that no agent can discover the color of any neighbor (let alone any non-neighboring agent) in the solution chosen to the problem.

\subsubsection {Previous Work on Privacy in DisCSP}
\label {sec:previous_privacy}

Before discussing what information may be leaked by a given algorithm, and how to prevent it, it is important to clarify what information is assumed to be initially known to each agent. 

\paragraph {Initial Knowledge Assumptions}

In this paper, we use the following three assumptions, which are currently the most widely used in the DisCSP literature. 
\begin{enumerate}
\item Each agent~$a$ knows all agents that own variables that are neighbors of $a$'s variables, but does not know any of the other agents (not even their existence); 
\item A variable and its domain are known only to its owner agent and to the agents owning neighboring variables, but the other agents ignore the existence of the variable; 
\item A constraint is fully known to all agents owning variables in its scope, and no other agent knows anything about the constraint (not even its existence). 
\end{enumerate}

\citet {Brito03} introduced \emph {Partially Known Constraints (PKCs}), whose scopes are known to all agents involved, but the knowledge of whose nature (which assignments are allowed or disallowed) is distributed among these agents. This is a relaxation of Assumption~3; however it is worth noting that the algorithms presented in this paper can still support PKCs without introducing privacy leaks by enforcing this assumption, because any PKC can be decomposed into a number of constraints over \emph{copy variables} such that Assumption~3 holds. For instance, if agents $a_1 \ldots a_n$ share the knowledge of a unary PKC over variable~$x$, then this constraint can be decomposed into $n$~unary constraints, such that each constraint~$c_i$ is known fully and only to agent~$a_i$ and is expressed over a copy variable~$x_i$ owned by $a_i$. Equality constraints are added to the problem to enforce equality of all copy variables. However, the introduction of copy variables can be detrimental to decision privacy. \citet {Grubshtein09} later proposed the similar concept of \emph {asymmetric} constraints, which can also be reformulated as symmetric constraints over copy variables for the purpose of applying our algorithms. 

Other previous work adopted a dual approach, assuming that variables are public and known to all agents, but each constraint is known to only one agent \cite {Silaghi00,Yokoo02a,Silaghi05a}. \citet {Silaghi05} even proposed a framework in which the constraints are secret to everyone. This dual approach has the disadvantage of necessarily violating topology privacy, since all variables are public.

\paragraph {Measuring Constraint Privacy Loss}

Most of the literature on privacy in DisCSPs focuses on constraint privacy. Metrics have been proposed to evaluate constraint privacy loss in algorithms, in particular for distributed meeting scheduling~\cite {Franzin04,Wallace05}. \citet {Maheswaran06} designed a framework called \emph {Valuation of Possible States (VPS)} that they used to measure constraint privacy loss in the OptAPO and SynchBB algorithms, and they considered the impact of whether the problem topology is public or only partially known to the agents. \citet {Greenstadt06} also applied VPS to evaluate DPOP and ADOPT on meeting scheduling problems, under the assumption that the problem topology is public. \citet {Doshi08} proposed to consider the cost of privacy loss in optimization problems, in order to elegantly balance privacy and optimality.

\paragraph {Preventing Constraint Privacy Loss}

Some previous work also proposed approaches to partially reduce constraint privacy loss. For instance, \citet {Brito07} described a modification of the \emph{Distributed Forward Checking (DisFC)} algorithm for DisCSPs in which agents are allowed to \emph{lie} for a finite time in order to achieve higher levels of privacy. However, the performance of most search-based algorithms like DisFC leaks information about constraint tightness, as explained at the end of Section~\ref {sec:DPOP}. To avoid this subtle privacy leak, one must either perform full exhaustive search, which is the option chosen by \citeauthor {Silaghi05}, or resort to Dynamic Programming, which is the option we have chosen in this paper. 

The cryptographic technique of \emph {secret sharing} \cite {Shamir79,Ben-Or88} was also applied by \citet {Silaghi06} and \citet {Greenstadt07} to lower constraint privacy in DPOP, assuming that the constraint graph topology is public knowledge. Cryptography has also been applied to provide strong guarantees on constraint privacy preservation in multi-agent decision making. For instance, \citet {Yokoo02}, \citet {Yokoo02a} and \citet {Yokoo05} showed how a public key encryption scheme can be used to solve DisCSPs using multiple servers, while protecting both constraint privacy and decision privacy. \citet {Bilogrevic11} solved single-meeting scheduling problems using similar techniques, and one semi-trusted server. In this paper however, we only consider algorithms that do not make use of third parties, as such third parties might not be available. \citet {Herlea01} showed how to use \emph {Secure Multiparty Computation (SMC)}\footnote {\citeauthor {Silaghi05} uses the different acronym \emph {MPC} for the same concept. } to securely schedule a single meeting, without relying on servers. In SMC, agents collaboratively compute the value of a given, publicly known function on private inputs, without revealing the inputs. For \citet {Herlea01}, the  inputs are each participant's availability at a given time, and the function outputs whether they are all available.

\paragraph {The MPC-DisCSP4 Algorithm}
\citet {Silaghi05a} also applied SMC to solve general DisCSPs, where the private inputs are the agents' constraint valuations, and the function returns a randomly chosen solution. The algorithm proceeds as follows \cite {Leaute11b}. Each agent~$a_i$ first creates a vector~$F_i$ with one entry per candidate solution to the DisCSP, equal to~$1$ if the candidate solution satisfies $a_i$'s private constraints, and to~$0$ otherwise. To reduce the size of~$F_i$, the candidate solutions may be filtered through publicly known constraints, if there exists any. Using Shamir's polynomial secret sharing technique \cite {Shamir79,Ben-Or88}, agent~$a_i$ then sends one secret share~$F_{i \rightarrow j}$ of its vector~$F_i$ to each other agent~$a_j$, and receives corresponding secret shares~$F_{j \rightarrow i}$ of their respective vectors. Agent~$a_i$ then multiplies together all the secret shares it received. The multiplication of Shamir secret shares is a non-trivial operation, because each secret share is the value of a polynomial, and multiplying two polynomials increases the degree of the output, which must always remain lower than the number $|\mathcal {A}|$ of agents to be resolvable. Therefore, after each multiplication of two secret shares, agent~$a_i$ must perform a complex sequence of operations involving the exchange of messages in order to reduce the degree of the output. 

After performing $(|\mathcal {A}| - 1)$ such pairwise multiplications of secret shares, agent~$a_i$'s vector~$F_i$ contains secret shares of~$1$ at the entries corresponding to globally feasible solutions. Agent~$a_i$ then performs a transformation on~$F_i$ so that only \emph {one} such secret share of~$1$ remains, identifying one particular feasible solution (if there exists one). Just selecting the first such entry would a posteriori reveal that all previous entries correspond to infeasible solutions to the DisCSP; to prevent this privacy leak, the vector~$F_i$ is first collaboratively, randomly permuted using a \emph {mix-net}. Agent~$a_i$ then performs a sequence of iterative operations on~$F_i$ (including communication-intensive multiplications) to set all its entries to secret shares of~$0$, except for one secret share of~$1$ corresponding to the chosen solution to the DisCSP (if any). The vector~$F_i$ is then un-shuffled by re-traversing the mix-net in reverse. Finally, agent~$a_i$ can compute secret shares of the domain index of each variable's chosen assignment, and reveal these secret shares only to the owners of the variables. 

This algorithm has numerous drawbacks. First, each agent must know all variables and their domains to construct its initial vector~$F_i$, which immediately violates agent privacy and topology privacy (Table~\ref {table:privacy_summary}, page~\pageref {table:privacy_summary}). Second, Shamir's secret sharing scheme is a majority threshold scheme, which means that if at least half of the agents collude, they can discover everyone's private information. Even though, in this paper, we are assuming that agents are honest and do not collude, a consequence of this threshold is that this scheme does not provide any privacy guarantee when the problem involves only two agents. Third, this algorithm is often only practical for very small problems, because it performs full exhaustive search; this is demonstrated by our experimental results in Section~\ref {sec:results}.

\section {P-DPOP$^+$: Full Agent Privacy and Partial Topology, Constraint and Decision Privacy}
\label {sec:P_DPOP}

This section describes a variant of the DPOP algorithm that guarantees full agent privacy. It also partially protects topology, constraint, and decision privacy. Algorithm~\ref {algo:P_DPOP} is an improvement over the \emph{P-DPOP} algorithm we originally proposed \cite {Faltings08a}. Like DPOP, the algorithm performs dynamic programming on a DFS-tree ordering of the variables (Figure~\ref {fig:dfs}). Algorithms to first elect one variable, and then generate a DFS tree rooted at this variable are given in Online Appendices 1 and~2. These algorithms do not reveal the pseudo-tree in its entirety to any agent; instead, each agent only discovers the (pseudo-)parents and (pseudo-)children of its own variables. For the sake of simplicity, we will hereafter assume without loss of generality that the constraint graph consists of a single component. If the problem actually consisted of two or more fully decoupled subproblems, then each subproblem would be solved in parallel, independently from the others. 

\begin {algorithm}[b!]
\begin {algorithmic}[1]
\REQUIRE a DFS-tree ordering of the variables

\STATE \COMMENT {Choose and exchange codenames for $x$ and its domain $D_x$:}
\STATE Wait for a message (CODES, $\tilde {y_i^x}, \tilde {D_{y_i}^x}, \sigma_{y_i}^x$) from each $y_i \in \{parent_x\} \cup pseudo\_parents_x$\label {algo:P_DPOP:codenames}
\FOR {each $y_i \in children_x \cup pseudo\_children_x$}
	\STATE $\tilde {x^{y_i}} \leftarrow $ large random number		\label {algo:P_DPOP:choose_var_codename}
	\STATE $\tilde {D_x^{y_i}} \leftarrow $ list of $|D_x|$ random, unique identifiers
	\STATE $\sigma_x^{y_i} \leftarrow $ random permutation of $[1, \ldots, |D_x|]$
	\STATE Send message (CODES, $\tilde {x^{y_i}}, \tilde {D_x^{y_i}}, \sigma_x^{y_i}$) to~$y_i$  		\label {algo:P_DPOP:send_codenames}
\ENDFOR

\vspace {5pt}
\STATE \COMMENT {Choose and exchange obfuscation key for $x$:}
\STATE Wait for and record a message (KEY, $key_{y_i}^x$) from each $y_i \in pseudo\_parents_x$ (if any) 	\label {algo:P_DPOP:got_keys}
\FOR {each $y_i \in pseudo\_children_x$}
	\STATE $key^{y_i}_x \leftarrow $ vector of large random numbers of $B$ bits, indexed by $D_x$ 	\label {algo:P_DPOP:key}
	\STATE Send message (KEY, $key^{y_i}_x$) to $y_i$ 	\label {algo:P_DPOP:send_keys}
\ENDFOR

\vspace {5pt}
\STATE Propagate feasibility values up the pseudo-tree (Algorithm~\ref {algo:UTILpropagation}, Section~\ref {sec:UTILpropagation}) 	\label {algo:P_DPOP:UTILpropagation}

\vspace {5pt}
\STATE \COMMENT {Propagate decisions top-down along the pseudo-tree (Section~\ref {sec:VALUEpropagation}):} 	\label {algo:P_DPOP:VALUEpropagation_start}
\IF {$x$ is not the root}
	\STATE Wait for message (DECISION, $\tilde {p_x^*}, \cdot$) from parent $p_x$
	\STATE $x^* \leftarrow x^*(\tilde {p_x} = \tilde {p_x^*}, \cdot)$ \COMMENT {where $x^*(\cdot)$ was computed in Algorithm~\ref {algo:UTILpropagation}, line~\ref {algo:UTILpropagation:argmin}} 	\label {algo:P_DPOP:lookup}
\ENDIF
\FOR {each $y_i \in children_x$}
	\STATE Send message (DECISION, $\tilde {sep_{y_i}^*}$) to $y_i$, with $sep_{y_i}$ from Algorithm~\ref {algo:UTILpropagation}, line~\ref {algo:UTILpropagation:separator} 	\label {algo:P_DPOP:VALUEpropagation_end}
\ENDFOR
\end {algorithmic}
\caption {Overal P-DPOP$^+$ algorithm, for variable~$x$}
\label {algo:P_DPOP}
\end {algorithm}

\subsection {Finding a Feasible Value for the Root Variable}
\label {sec:UTILpropagation}

As already illustrated for DPOP in Section~\ref {sec:DPOP}, the agents perform a bottom-up propagation of feasibility values along the pseudo-tree. This is done in Algorithm~\ref {algo:UTILpropagation}, which is an extension of DPOP's UTIL propagation phase (the extensions are indicated by comments in bold), and improves over the algorithm we originally proposed \cite {Faltings08a} by patching an important constraint privacy leak in the single-variable FEAS messages sent by variables with singleton separators. The following sections describe the obfuscation techniques used to protect the private information that could be leaked by the feasibility messages, using codenames (Section~\ref {sec:codenames}) and addition of random numbers (Section~\ref {sec:obfuscation}).

\begin {algorithm}[b!]
\begin{algorithmic}[1]

\REQUIRE a DFS-tree ordering of the variables; $p_x$ denotes $x$'s parent
\STATE \COMMENT {Join local constraints:} 	\label {algo:UTILpropagation:start}
\STATE $m(x, p_x, \cdot) \leftarrow \Sigma_{c \in \left\{ c' \in \mathcal {C} ~|~ x \in scope(c') ~\wedge~ scope(c') \cap \left( children_x \cup pseudo\_children_x \right) = \emptyset \right\}} c (x, \cdot)$ 	\label {algo:UTILpropagation:local_join}

\vspace{5pt}
\STATE \COMMENT {\textbf {Apply codenames:}}
\FOR {each $y_i \in \{p_x\} \cup pseudo\_parents_x$}
	\STATE $m( x, \tilde {p_x}, \cdot ) \leftarrow$ replace $(y_i, D_{y_i})$ in $m( x, p_x, \cdot )$ with $(\tilde {y_i^x}, \tilde {D_{y_i}^x})$ from Algorithm~\ref {algo:P_DPOP}, line~\ref {algo:P_DPOP:codenames}, and apply the permutation~$\sigma_{y_i}^x$ to~$\tilde {D_{y_i}^x}$ 	\label {algo:UTILpropagation:codenames}
\ENDFOR

\vspace{5pt}
\STATE \COMMENT {\textbf{Obfuscate infeasible entries:}}
\STATE $r \leftarrow $ large, positive, random number of $B$ bits
\STATE $m(x,  \tilde {p_x}, \cdot ) \leftarrow 
\left\{
\begin{array}{lcl}
m(x, \tilde {p_x}, \cdot) 		& if 	& m(x, \tilde {p_x}, \cdot) = 0   \\
m(x, \tilde {p_x}, \cdot) + r 	& if 	& m(x, \tilde {p_x}, \cdot) > 0
\end{array}
\right.$ 		\label {algo:UTILpropagation:new_obfuscation}

\vspace{5pt}
\STATE \COMMENT {Join with received messages:}
\FOR {each $y_i \in children_x$}
	\STATE Wait for the message (FEAS, $m_i(\tilde{x}, \cdot)$) from $y_i$ 		\label {algo:UTILpropagation:get_message}
	\STATE $sep_{y_i} \leftarrow scope(m_i)$ 	\label {algo:UTILpropagation:separator}
	\FOR [\textbf {resolve codenames}] {each $z \in children_x \cup pseudo\_children_x$}
		\STATE $m_i(x, \cdot) \leftarrow $ identify $(\tilde{x^{z}}, \tilde {D_x^{z}})$ as $(x, D_x)$ in $m_i(\tilde{x}, \cdot)$ (if $\tilde{x^{z}}$ is present) 		\label {algo:UTILpropagation:resolve_codenames}
	\ENDFOR
	\STATE $m( x, \tilde {p_x}, \cdot ) \leftarrow m(x, \tilde {p_x}, \cdot) + m_i(x, \cdot )$ 	\label {algo:UTILpropagation:join} 
\ENDFOR

\vspace{5pt}
\STATE \COMMENT {\textbf{De-obfuscate feasibility values with respect to $x$:}}
\FOR {each $y_i \in pseudo\_children_x$}
	\STATE $m( x, \tilde {p_x}, \cdot ) \leftarrow m( x, \tilde {p_x}, \cdot ) - key_x^{y_i} (x)$ \COMMENT {with $key_x^{y_i}$ from Algorithm~\ref {algo:P_DPOP}, line~\ref {algo:P_DPOP:key}}		\label {algo:UTILpropagation:deobfuscation}
\ENDFOR

\vspace{5pt}
\STATE \COMMENT {Project out $x$:}
\IF{$x$ is not the root variable}

	\STATE $x^*(\tilde {p_x}, \cdot) \leftarrow \arg\min_x \left\{ m( x, \tilde {p_x}, \cdot ) \right\}$ \label {algo:UTILpropagation:argmin}
	\STATE $m(\tilde {p_x}, \cdot) \leftarrow \min_x \left\{ m( x, \tilde {p_x}, \cdot ) \right\}$ 	\label {algo:UTILpropagation:min}
	
	\vspace{5pt}
	\STATE \COMMENT {\textbf{Obfuscate feasibility values}:}
	\FOR {each $y_i \in pseudo\_parents_x$}
		\STATE $m( \tilde {p_x}, \cdot ) \leftarrow m( \tilde {p_x}, \cdot ) + key_{y_i}^x (\tilde {y_i^x})$ \COMMENT {with $key_{y_i}^x$ from Algorithm~\ref {algo:P_DPOP}, line~\ref {algo:P_DPOP:got_keys}} 		\label {algo:UTILpropagation:obfuscation}
	\ENDFOR

	\vspace{5pt}
	\STATE Send the message (FEAS, $m( \tilde {p_x}, \cdot )$) to $p_x$
\ENDIF
\STATE \textbf {else} $x^* \leftarrow \arg\min_x \left\{ m( x ) \right\}$ \label {algo:UTILpropagation:root_value} // $m( x, \tilde {p_x}, \cdot )$ actually only depends on $x$
\end{algorithmic}
\caption{Algorithm to find a feasible value for the root of a DFS tree, for variable~$x$}
\label{algo:UTILpropagation}
\end {algorithm}

\subsubsection {Hiding Variable Names and Values Using Codenames}
\label {sec:codenames}

Consider the feasibility message $x_1 \rightarrow x_4$ sent by agent~$a(x_1)$ to its parent variable~$x_4$ in Figure~\ref {fig:dfs}. This message is recalled in Figure~\ref {fig:msg_to_c33:cleartext}, reformulated in terms of minimizing the number of constraint violations. If this message were actually received in cleartext, it would breach agent privacy and topology privacy: agent~$a(x_4)$ would be able to infer from the dependency of the message on variable~$x_2$ both the existence of agent~$a(x_2)$ (which violates agent privacy) and the fact that $x_2$ is a neighbor of one or more unknown nodes below~$x_1$. 

\begin{figure}[ht]
\begin {center}
\subfigure [in cleartext] {
\begin{tabular}{|c||c|c|c|}
\multicolumn {4}{c}{$x_1 \rightarrow x_4$}	\\
\hline
		&	\multicolumn{3}{c|}{$x_2$}	\\
$x_4$	&	$R$		&	$B$		& 	$G$ 	\\
\hline \hline
$R$		&	$0$		&	$0$		&	$0$	\\
\hline
$B$		&	$0$		&	$0$		&	$1$	\\
\hline
$G$		&	$0$		&	$1$		&	$0$	\\
\hline
\end{tabular}
\label {fig:msg_to_c33:cleartext}
}
\subfigure [partly obfuscated] {
\begin{tabular}{|c||c|c|c|}
\multicolumn {4}{c}{$x_1 \rightarrow x_4$}	\\
\hline
		&	\multicolumn{3}{c|}{$928372$}	\\
$x_4$	&	$\alpha$		&	$\beta$		& 	$\gamma$ 	\\
\hline \hline
$R$		&	$0$		&	$0$		&	$0$	\\
\hline
$B$		&	$0$		&	$0$		&	$1$	\\
\hline
$G$		&	$0$		&	$1$		&	$0$	\\
\hline
\end{tabular}
\label {fig:msg_to_c33:cyphertext}
}
\subfigure [fully obfuscated] {
\begin{tabular}{|c||c|c|c|}
\multicolumn {4}{c}{$x_1 \rightarrow x_4$}	\\
\hline
		&	\multicolumn{3}{c|}{$928372$}	\\
$x_4$	&	$\alpha$		&	$\beta$		& 	$\gamma$ 	\\
\hline \hline
$R$		&	$620961$		&	$983655$		&	$534687$		\\
\hline
$B$		&	$620961$		&	$983655$		&	$534688$		\\
\hline
$G$		&	$620961$		&	$983656$		&	$534687$		\\
\hline
\end{tabular}
\label {fig:msg_to_c33:cyphertext2}
}
\caption {The message sent by agent~$a(x_1)$ to its parent variable~$x_4$ in Figure~\ref {fig:dfs}. }
\label{fig:msg_to_c33}
\end {center}
\end{figure}

In order to patch these privacy leaks, variable~$x_2$ and its domain~$D_2 = \{ R, B, G \}$ are replaced with random codenames $\tilde {x_2^{x_1}} = 928372$ and $\tilde {D_2^{x_1}} = \{ \alpha, \beta, \gamma \}$ (Figure~\ref {fig:msg_to_c33}b) preliminarily generated by~$a(x_2)$ and communicated directly to the leaf of the back-edge (Algorithm~\ref {algo:P_DPOP}, lines \ref {algo:P_DPOP:codenames} to~\ref {algo:P_DPOP:send_codenames}). The leaf applies these codenames to its output message (Algorithm~\ref {algo:UTILpropagation}, line~\ref {algo:UTILpropagation:codenames}), and they are only resolved once the propagation reaches the root of the back-edge (Algorithm~\ref {algo:UTILpropagation}, line~\ref {algo:UTILpropagation:resolve_codenames}). Not knowing these codenames, the agents in between, such as~$a(x_4)$, can only infer the existence of a cycle in the constraint graph involving some unknown ancestor and descendent. This is tolerated by the definition of topology privacy (Definition~\ref {def:topology}) since they are also involved in this cycle. A secret, random permutation~$\sigma_2^{x_1}$ is also applied to~$\tilde {D_2^{x_1}}$; this is useful for problem classes in which variable domains are public. Notice that if $x_4$ also had a constraint with~$x_2$, the above reasoning would still hold, because $x_2$ would then have sent a \emph {different} codename $\tilde {x_2^{x_4}}$ to~$x_4$, which would then not be able to resolve the unknown codename~$\tilde {x_2^{x_1}}$ to~$x_2$. In this case, $x_4$'s separator would be $\{ x_3, \tilde {x_2^{x_1}}, \tilde {x_2^{x_4}} \}$, and its message sent to~$x_3$ would be three-dimensional instead of two-dimensional.

\subsubsection {Obfuscating Feasibility Values}
\label {sec:obfuscation}

Hiding variable names and values using codenames addresses the leaks of agent and topology privacy. However, this does not address the fact that the feasibility values in the message $x_1 \rightarrow x_4$ in Figure~\ref {fig:msg_to_c33:cyphertext} violate constraint privacy, because they reveal to $x_4$ that its subtree can always find a feasible solution to its subproblem when $x_4 = R$, regardless of the value of the obfuscated variable~$928372$. To patch this privacy leak, feasibility values are obfuscated by adding large, random numbers that are generated by the root of the back-edge ($x_2$) and sent over a secure channel to the leaf of the back-edge (Algorithm~\ref {algo:P_DPOP}, lines \ref {algo:P_DPOP:got_keys} to~\ref {algo:P_DPOP:send_keys}). 
The number of bits~$B$ of the random numbers is a problem-independent parameter of the algorithm. 
The obfuscation is performed in such a way that a different random number is added to all feasibility values associated with each value of~$x_2$, as in Figure~\ref {fig:msg_to_c33:cyphertext2}, using the obfuscation key $[ 620961, 983655, 534687 ]$. These random numbers are added by the leaf of the back-edge to its outgoing message (Algorithm~\ref {algo:UTILpropagation}, line~\ref {algo:UTILpropagation:obfuscation}), and they are only eventually subtracted when the propagation reaches the root of the back-edge (Algorithm~\ref {algo:UTILpropagation}, line~\ref {algo:UTILpropagation:deobfuscation}). 

Notice that this obfuscation scheme achieves two objectives: 1) it hides from $x_4$ the absolute feasibility values of its subtree, and 2) it hides the relative dependencies of these values on the obfuscated variable~$928372$, because different random numbers were used for each value in its obfuscated domain $\{ \alpha, \beta, \gamma \}$. Agent~$a(x_4)$ is still able to infer the relative dependencies on its own variable~$x_4$, which is necessary to perform the projection of this variable, but it is unable to tell, for each value of the other (obfuscated) variable, whether the subtree's problem is feasible, and if not, how many constraints are violated. Notice in particular that, for a given value of the obfuscated variable (i.e. for any column), agent~$a(x_4)$ does not know whether any of the assignments to~$x_4$ is feasible, and therefore it would be incorrect to simply assume that the lowest of the obfuscated feasibility entries decrypts to~$0$. Similarly, equal entries in the same column correspond with a high probability to entries that have the same number of constraint violations, but this number is not necessarily~0, so it would be incorrect to infer they correspond to feasible entries. 

Notice also that this obfuscation scheme is only applicable in the presence of a back-edge, i.e. when the message contains more than just the parent variable. Consider for instance the single-variable message $x_5 \rightarrow x_3$, recalled in Figure~\ref {fig:msg_from_c21:cleartext}. If agent~$a(x_3)$ knew that $x_5$ is a leaf of the pseudo-tree, the cleartext message would reveal agent~$a(x_5)$'s private local constraint $x_5 \not\in \{ B, R \}$ to agent~$a(x_3)$, and the previous obfuscation scheme does not apply because of the absence of back-edges. Notice that this threat to constraint privacy is tempered by the fact that P-DPOP$^+$'s guarantees in terms of topology privacy prevent agent~$a(x_3)$ from discovering that $x_5$ is indeed a leaf. From $a(x_3)$'s point of view, a larger subproblem might be hanging below variable~$x_5$ in Figure~\ref {fig:dfs}, and the message could actually be an aggregation of multiple agents' subproblems. 

\begin{figure}[ht]
\begin {center}
\subfigure [in cleartext] {
\begin{tabular}{|c||c|}
\multicolumn {2}{c}{$x_5 \rightarrow x_3$}	\\
\hline
$x_3$	&	\# conflicts \\
\hline \hline
$R$		&	$0$	\\
\hline
$B$		&	$0$		\\
\hline
$G$		&	$1$		\\
\hline
\end{tabular}
\label {fig:msg_from_c21:cleartext}
}
\subfigure [obfuscated] {
\begin{tabular}{|c||c|}
\multicolumn {2}{c}{$x_5 \rightarrow x_3$}	\\
\hline
$x_3$	&	\# conflicts \\
\hline \hline
$R$		&	$0$	\\
\hline
$B$		&	$0$		\\
\hline
$G$		&	$730957$		\\
\hline
\end{tabular}
\label {fig:msg_from_c21:cyphertext}
}
\caption {The message received by agent~$a(x_3)$ in Figure~\ref {fig:dfs}. }
\label{fig:msg_from_c21}
\end {center}
\end{figure}

To reduce this privacy leak present in the original algorithm \cite {Faltings08a}, we propose a new additional obfuscation scheme that consists in adding large ($B$-bit), positive, random numbers to positive entries in single-variable messages, in order to obfuscate the true numbers of constraint violations (Algorithm~\ref {algo:UTILpropagation}, line~\ref {algo:UTILpropagation:new_obfuscation} and Figure~\ref {fig:msg_from_c21:cyphertext}). Because these random numbers are never subtracted back, they must not be added to zero entries, otherwise the algorithm would fail to find a solution with no violation. Feasible entries are still revealed, but the numbers of constraint violations for infeasible entries remain obfuscated.

\subsection {Propagating Final Decisions}
\label {sec:VALUEpropagation}

Once the feasibility values have propagated up all the way to the root of the pseudo-tree, and a feasible assignment to this root variable has been found (if there exists one), this assignment is propagated down the pseudo-tree (Algorithm~\ref {algo:P_DPOP}, lines \ref {algo:P_DPOP:VALUEpropagation_start} to~\ref {algo:P_DPOP:VALUEpropagation_end}). Each variable uses the assignments contained in the message from its parent, in order to look up a corresponding assignment for itself (line~\ref {algo:P_DPOP:lookup}). It then sends to each child the assignments for the variables in its separator (line~\ref {algo:P_DPOP:VALUEpropagation_end}), using the same codenames as before so as to protect agent and topology privacy. Decision privacy is only partially guaranteed, because each variable learns the values chosen for its parent and pseudo-parents --- but not for other, non-neighboring variables in its separator because they are hidden by unknown codenames.

\subsection {Algorithm Properties}
\label {sec:P_DPOP_properties}

This section first formally proves that the algorithm is complete, and analyses its complexity. We then present an algorithm variant with lower complexity. Finally, the privacy guarantees provided by both algorithms (summarized in Table~\ref {table:privacy_summary}) are formally described. 

\begin{table}[ht]
\begin{center}
\begin{tabular}{|l||c|c|c|c|}
\hline
privacy type: 					& agent		 		& topology	& constraint	 		& decision \\
\hline
\hline
P-DPOP$^{(+)}$				& \cellcolor[gray]{0.8}full	& partial 		& partial 				& partial \\
\hline
P$^{\sfrac{3}{2}}$-DPOP$^{(+)}$ 	& \cellcolor[gray]{0.8}full	& partial 		& partial 				& \cellcolor[gray]{0.8}full	 \\
\hline
P$^2$-DPOP$^{(+)}$ 			& \cellcolor[gray]{0.8}full	& partial 		& \cellcolor[gray]{0.8}full	& \cellcolor[gray]{0.8}full	 \\
\hline
\hline
MPC-DisCSP4		 			& -					& partial 		& partial 				& partial \\
\hline
\end{tabular}
\label{table:privacy_summary}
\end{center}
\caption{Privacy guarantees of various algorithms. }
\end{table}

\subsubsection {Completeness and Complexity}

\begin {theorem}
\label {thm:P_DPOP}
Provided that there are no codename clashes, P-DPOP$^+$ (Algorithm~\ref {algo:P_DPOP}) terminates and returns a feasible solution to the DisCSP, if there exists one. 
\end {theorem}

\begin {proof}
After exchanging codenames and obfuscation keys, which is guaranteed to require a number of messages at most quadratic in the number~$n$ of variables, the bottom-up propagation of feasibility values (Algorithm~\ref {algo:UTILpropagation}) terminates after sending exactly $(n-1)$ messages (one up each tree-edge). One can prove by induction (left to the reader) that this multi-party dynamic programming computation almost surely correctly reveals to each variable~$x$ the (obfuscated) feasibility of its subtree's subproblem, as a function of~$x$ and possibly of ancestor variables in the pseudo-tree. This process may only fail in case of collisions of codenames, when the roots of two overlapping back-edges choose the same codenames. Such codename clashes are inherent to most privacy-protecting algorithms, and can be made as improbable as desired by augmenting the size of the codename space. 

Finally, the top-down decision propagation phase (Algorithm~\ref {algo:P_DPOP}, lines \ref {algo:P_DPOP:VALUEpropagation_start} to~\ref {algo:P_DPOP:VALUEpropagation_end}) is guaranteed to yield a feasible assignment to each variable (if there exists one), after the exchange of exactly $(n-1)$ messages (one down each tree-edge). 
\end {proof}

When it comes to the complexity of the algorithm in terms of number of messages exchanged, the bottleneck is in the election of the root variable (Online Appendix~1), which requires $O(\phi \cdot d \cdot n^2)$ messages, where $\phi$ is the diameter of the constraint graph, $d$~its degree, and $n$~is the number of variables. However, the $(n-1)$ messages containing feasibility values can be exponentially large: the message sent by variable~$x$ is expressed over $|sep_x|$ variable codenames (Algorithm~\ref {algo:UTILpropagation}, line~\ref {algo:UTILpropagation:separator}), and therefore contains $O(D_{\max}^{|sep_x|})$ feasibility values, where $D_{\max}$ is the size of the largest variable domain. The overall complexity in terms of information exchange, memory and runtime (measured in number of constraint checks) is therefore $O(n \cdot D_{\max}^{sep_{\max}})$, where $sep_{\max} = \max_x |sep_x|$. This is the same as DPOP, except that in P-DPOP$^+$ each variable may appear multiple times under different codenames in the same separator, hereby increasing the value of~$sep_{\max}$. However, this increase is only by a multiplicative factor that is upper bounded by the degree of the constraint graph, since the number of codenames for a given variable is at most equal to its number of neighbors. Empirically, our experimental results in Section~\ref {sec:results} suggest that, on almost all the problem classes we considered, the median value of $sep_{\max}$ tends to grow rather linearly in~$n$.

\subsubsection {P-DPOP: Trading off Topology Privacy for Performance}
\label {sec:mergeBack}

It is possible to reduce the sizes~$|sep_{x_i}|$ of the separators, by enforcing that each agent~$a(x)$ send the \emph {same} codename~$\tilde {x}$ for~$x$ to \emph {all} of $x$'s (pseudo-)children, unlike in Algorithm~\ref {algo:P_DPOP} (lines \ref {algo:P_DPOP:codenames} to~\ref {algo:P_DPOP:send_codenames}). This variant will be identified by the absence of the plus sign in exponent; P-DPOP is the version of the algorithm that was initially proposed by \citet{Faltings08a}. 

As a result of this change, variables that previously may have occurred multiple times in the same feasibility message under different codenames can now only appear at most once, such that we now have $sep_{\max} < n$. The worst-case complexity of P-DPOP then becomes the same as DPOP \cite {Petcu05}, in which $sep_{\max}$ is equal to the \emph {width} of the pseudo-tree, which is bounded below by the \emph {treewidth} of the constraint graph. However, privacy considerations prevents the use in P-DPOP of DPOP's more efficient, but less privacy-aware pseudo-tree generation heuristics, resulting in higher-width pseudo-trees.  

While the complexity of P-DPOP is hereby decreased compared to P-DPOP$^+$, sending the same codename~$\tilde {x}$ for variable~$x$ to all its (pseudo-)children has drawbacks in terms of topology privacy, as analyzed below.

\subsubsection {Full Agent Privacy}
\label {sec:agent_privacy_proof:P_DPOP}

There are only two ways the identity of an agent~$A$ could be leaked to a non-neighbor~$B$: 1) the algorithm can require $A$ and~$B$ to exchange messages with each other, or 2) Agent~$A$ can receive a message whose content refers identifiably to~$B$. Case~1 can never happen in any of our algorithms, because they only ever involve exchanging messages with neighboring agents. Case~2 is addressed mainly through the use of codenames. 

\begin {theorem}
\label {thm:agent_privacy:P_DPOP}
The P-DPOP$^{(+)}$ algorithms guarantee full agent privacy. 
\end {theorem}

\begin {proof}
The P-DPOP$^{(+)}$ algorithms proceed in the following sequential phases (the preliminary phases of root election and pseudo-tree generation are addressed in online appendices): 
\begin{description}
\item [Bottom-up feasibility propagation (Algorithm~\ref {algo:UTILpropagation})] Each feasibility message contains only a function (line~\ref {algo:UTILpropagation:get_message}) over a set of variables, whose names, if transmitted in clear text, could identify their owner agent. To prevent this agent privacy leak,  P-DPOP$^{(+)}$ replaces all variable names with secret, random codenames, as follows. 

Consider a variable~$x$ in the pseudo-tree. Note that no feasibility message sent by~$x$ or by any ancestor of~$x$ can be a function of~$x$. The message sent by~$x$ is not a function of~$x$, because $x$ is projected out before the message is sent (line~\ref {algo:UTILpropagation:min}). Variable~$x$ cannot re-appear in any feasibility message higher in the pseudo-tree, because no agent's local problem can involve any variable lower in the pseudo-tree (line~\ref {algo:UTILpropagation:local_join}). 

Similarly, consider now the feasibility message sent by a descendant~$y$ of~$x$ in the pseudo-tree, and assume first that $y$ is a leaf of the pseudo-tree. Since $y$ has no children, the feasibility message it sends can only be a function of the variables in its local problem. If this local problem involves~$x$, $y$~will replace $x$ by its codename~$\tilde {x^y}$ (line~\ref {algo:UTILpropagation:codenames}) before it sends its feasibility message. One can then prove by inference that no feasibility message sent by any variable between $y$ and~$x$ will contain~$x$ either; it can only (and not necessarily) contain one or several of its codenames~$\tilde {x^{y_i}}$. 

Since the codenames~$\tilde {x^{y_i}}$ are random numbers chosen by~$x$ (Algorithm~\ref {algo:P_DPOP}, line~\ref {algo:P_DPOP:choose_var_codename}), and only communicated (through channels that are assumed secure) to the respective neighbors~$y_i$ of~$x$ (Algorithm~\ref {algo:P_DPOP}, line~\ref {algo:P_DPOP:send_codenames}), no non-neighbor of~$x$ receiving a message involving any~$\tilde {x^{y_i}}$ can discover the identity of its owner agent. 

The domain~$D_x$ of variable~$x$ could also contain values that might identify its owner agent. To fix this privacy risk, $x$'s domain is also replaced by obfuscated domains~$\tilde {D_x^{y_i}}$ of random numbers, similarly to the way variable names are obfuscated. In this paper, we make the simplifying assumption that all variables have the same domain size (which naturally holds in many problem classes), so that one variable's domain size does not give any information about its owner agent. Otherwise, variable domains can be padded with fake values in order to make them all have the same size. 

\item [Top-down decision propagation (Section~\ref {sec:VALUEpropagation})] The messages contain assignments to variables (Algorithm~\ref {algo:P_DPOP}, line~\ref {algo:P_DPOP:VALUEpropagation_end}), which are also obfuscated using codenames. 
\end{description}

This concludes the proof that, in the P-DPOP$^{(+)}$ algorithms, no agent can receive any message from which it can infer the identity of any non-neighboring agent. 
\end {proof}

\subsubsection {Partial Topology Privacy}
\label {sec:topo:p_dpop}

\begin {theorem}
\label {thm:topo_privacy:P_DPOPmin}
P-DPOP guarantees \emph {partial} topology privacy. The minor leaks of topology privacy lie in the fact that a variable \emph {might} be able to discover a lower bound on a neighbor variable's degree in the constraint graph, and a lower bound on the total number of variables. 
\end {theorem}

\begin {proof}
Root election and pseudo-tree generation are left to the online appendices. 

\begin{description}
\item [Bottom-up feasibility propagation (Algorithm~\ref {algo:UTILpropagation})] 
Each variable~$x$ receives a FEAS message from each child, containing a function whose scope might reveal topological information. Each variable~$y$ in this scope is represented by a secret codename~$\tilde{y}$, however $x$ may be able to decrypt the codename~$\tilde {y}$, if and only if $y$ is a neighbor of~$x$ (or is $x$ itself), because $y$ has sent the \emph {same} codename~$\tilde {y}$ to all its neighbors. This results in a leak of topology privacy: $x$ discovers, for each neighboring ancestor~$y$, whether $y$ has at least one other neighbor below a given child of~$x$. But it cannot discover exactly how many of these other neighbors there are.

Furthermore, in the case where $x$ and $y$ are not neighbors, $x$ cannot decrypt~$\tilde {y}$, but it can still infer there exists another, non-neighboring ancestor corresponding to this codename. This is another breach of topology privacy. Because $y$ sent the same codename~$\tilde {y}$ to all its neighbors, $x$ can also discover whether that other ancestor has at least one neighbor below each of $x$'s children. Moreover, since codenames are large random numbers that are almost surely unique, $x$ may discover the existence of \emph {several, distinct} such non-neighboring ancestors. 

\item [Top-down decision propagation (Section~\ref {sec:VALUEpropagation})] Each variable receives a message from its parent, which can only contain codenames for variables and variable values that were already present in the FEAS message received during the previous phase. 
\end{description}
This concludes the proof that P-DPOP only \emph {partially} protects topology privacy. The limited topology information leaked to a variable only concerns its branch in the pseudo-tree; no information can be leaked about any other branch, not even their existence. 
\end {proof}

\begin {theorem}
The use of different codenames for each (pseudo-)child improves the topology privacy in P-DPOP$^+$ compared to P-DPOP, but the same bounds can still be leaked. 
\end {theorem}

\begin {proof}
Consider a variable~$x$ that receives a FEAS message including a secret codename~$\tilde {y}$ corresponding to variable~$y$ ($\neq x$). Because $y$ now sent a \emph {different} codename to each of its neighbors, $x$ is no longer able to decrypt~$\tilde {y}$, even if $y$ is a neighbor of~$x$. As a consequence, $x$ is no longer able to infer whether $\tilde {y}$ refers to a known neighbor of~$x$, or to an unknown, non-neighboring variable. However, since each codename now corresponds to a unique back-edge in the pseudo-tree, for each pair $(\alpha, \beta)$ of unknown codenames in $x$'s received FEAS message (if such a pair exists), at least one of the following statements must hold: 
\begin{itemize}
\item $\alpha$ and $\beta$ refer to two different ancestors of~$x$, and therefore $x$ discovers at it has at least two ancestors (which it might not have known, if it has no pseudo-parent); \emph {and/or}
\item $\alpha$ and $\beta$ were sent to two different descendants of~$x$ below (and possibly including) the sender child~$y$, and therefore $x$ discovers that it has at least two descendants below (and including)~$y$ (which it might not have known, if it has no pseudo-child below~$y$). 
\end{itemize}
Therefore $x$ \emph {might} be able to refine its lower bound on the total number of variables. 
\end {proof}

\subsubsection {Partial Constraint Privacy}
\label {sec:constraint_privacy:obfuscation}

\begin {theorem}
The P-DPOP$^{(+)}$ algorithms guarantee \emph {partial} constraint privacy. The local feasibility of a subproblem for a partial variable assignment~$X^*$ may be leaked, even if $X^*$ cannot be extended to an overall feasible solution (i.e. this is not semi-private information). 
\end {theorem}

\begin{proof}
Information about constraints is only transmitted during feasibility propagation (Algorithm~\ref {algo:UTILpropagation}). Based on the knowledge of the optimal variable assignments transmitted during the last phase (Section~\ref {sec:VALUEpropagation}), some of the feasibility information may be decrypted. 

\begin{description}
\item [Single-variable feasibility messages] 
When a variable~$p_x$ receives a feasibility message involving only~$p_x$, the message has been obfuscated only by adding secret random numbers to its infeasible entries (line~\ref {algo:UTILpropagation:new_obfuscation}). Feasible entries remain equal to~$0$, and $p_x$ can identify which entries refer respectively to feasible or infeasible assignments to~$p_x$. 

However, the addition of a secret, positive, random number to each infeasible entry ensures only an upper bound on the number of constraint violations is leaked, which can be made as loose as desired by choosing random numbers as large as necessary.  

\item [Multi-variable feasibility messages] 
If the FEAS message involves at least one other variable~$\tilde {y_i}$, then all message entries have been obfuscated by adding large random numbers~$key_{y_i}^x (\tilde {y_i^x})$ of $B$ bits (line~\ref {algo:UTILpropagation:obfuscation}). Furthermore, $key_{y_i}^x (\tilde {y_i^x})$ is only known to the sender~$x$ of the message and to its pseudo-parent~$y_i$, but not to the recipient~$p_x$, which therefore cannot subtract it to de-obfuscate the entries. 

Assume, for simplicity, that the message $m(\tilde {p_x}, \tilde {y_i^x})$ involves only the two variables $\tilde {p_x}$ and~$\tilde {y_i^x}$; the argument extends easily to more variables. The recipient~$p_x$ might be able to make inferences: 1) by fixing $\tilde {y_i^x}$ and comparing the obfuscated entries corresponding to different values for~$\tilde {p_x}$; or 2) by fixing~$\tilde {p_x}$ and varying~$\tilde {y_i^x}$ instead. 
\begin{enumerate}
\item For a given value of~$\tilde {y_i^x}$, all entries have been obfuscated by adding the same random number~$key_{y_i}^x (\tilde {y_i^x})$ (line~\ref {algo:UTILpropagation:obfuscation}), so $p_x$ can compute the relative differences of feasibility values for various assignments to~$\tilde {p_x}$. However, it cannot decrypt the absolute values without knowing~$key_{y_i}^x (\tilde {y_i^x})$. In particular, the lowest obfuscated value is not necessarily equal to~$key_{y_i}^x (\tilde {y_i^x})$, because it does not necessarily decrypt to~$0$: all values of~$\tilde {p_x}$ may be infeasible for this particular value of~$\tilde {y_i^x}$. 

There is one exception: if a feasible solution is found to the problem in which $\tilde {y_i^x} = \tilde {y_i^x}^*$ and $\tilde {p_x} = \tilde {p_x}^*$, then $m(\tilde {p_x}^*, \tilde {y_i^x}^*)$ necessarily decrypts to~$0$, and therefore $p_x$ will be able to infer~$key_{y_i}^x (\tilde {y_i^x}^*)$. After fixing $\tilde {y_i^x} = \tilde {y_i^x}^*$ in the message and subtracting~$key_{y_i}^x (\tilde {y_i^x}^*)$, the same reasoning can be made as for the single-variable case, in which feasible and infeasible entries are identifiable, but the numbers of constraint violations for infeasible entries remain obfuscated. 

\item For a given value of~$\tilde {p_x}$, each feasibility value $m(\tilde {p_x}, \tilde {y_i^x})$ has been obfuscated by adding a different, secret random number~$key_{y_i}^x (\tilde {y_i^x})$. Choosing the number of bits~$B$ sufficiently large makes sure that no useful information (relative, or absolute) can be obtained by comparing the obfuscated feasibility values. 
\end{enumerate}
\end{description}
This concludes the proof that P-DPOP$^{(+)}$ guarantees \emph {partial} constraint privacy. 
\end{proof}

\subsubsection {Partial Decision Privacy}

\begin {theorem}
The P-DPOP$^{(+)}$ algorithms guarantee \emph {partial} decision privacy. The leak lies in the fact that a variable \emph {might} discover the values chosen for some or all of its neighbors. 
\end {theorem}

\begin {proof}
First notice that the algorithm cannot leak any information about the chosen values for variables that are \emph {lower} in the pseudo-tree, since these variables have been projected out of the feasibility messages received. However, during the decision propagation phase, each variable receives a message from its parent that contains the chosen values for its parent and pseudo-parents. The message may also contain codenames for the assignments to other, non-neighboring variables, which the recipient will not be able to decode. Furthermore, domains are shuffled using secret permutations, making it impossible to decode the codename for the value of a non-neighboring variable from its index in the variable's domain. 
\end {proof}

\section {P$^{\sfrac{3}{2}}$-DPOP$^+$: Adding Full Decision Privacy}
\label {sec:P_DPOP_value}

This section presents another variant of the P-DPOP$^+$ algorithm that achieves \emph{full} decision privacy. This results in a novel algorithm, which can be seen as a hybrid between the P-DPOP$^+$ and P$^2$-DPOP~\cite {Leaute09a} algorithms, and is called \emph {P$^{\sfrac{3}{2}}$-DPOP$^+$}.

\begin {algorithm}[b!]
\begin {algorithmic}[1]
\REQUIRE a first temporary DFS tree, a unique ID~$id_x$, a tight strict lower bound on the next unique ID~$id_x^+$, and an upper bound $n^+$ on the total number of variables

\STATE $vector_x \leftarrow [ \underbrace {\overbrace {1, \ldots, 1}^{id_x}, 0, \overbrace {-1, \ldots, -1}^{id_x^+ - id_x}, 1, \ldots, 1}_{n^+} ]$ 	\label {algo:P_DPOP_value:vector}

\vspace {5pt}
\STATE \COMMENT {Exchange public key shares:} 		\label {algo:P_DPOP_value:keys_start}
\STATE $private_x \leftarrow $ generate a private ElGamal key for $x$
\STATE $public_x \leftarrow $ generate a set of $(id_x^+ - id_x + 1)$ public key shares corresponding to $private_x$
\STATE \textbf {for each} $share \in public_x$ \textbf {do} \textsc {ToPrevious}((SHARE, $share$)) as in Algorithm~\ref {algo:routing}
\FOR {$i = 1 \ldots n^+$}
	\STATE Wait for and record one message (SHARE, $share$)
	\STATE \textbf {if} $share \not\in public_x$ \textbf {then} \textsc {ToPrevious}((SHARE, $share$)) as in Algorithm~\ref {algo:routing}
\ENDFOR
\STATE Generate the compound ElGamal public key based on all the public key shares 	\label {algo:P_DPOP_value:keys_end}

\vspace {5pt}
\WHILE {$vector_x \neq \emptyset$} 	\label {algo:P_DPOP_value:while}
	\STATE Choose a new root (Algorithm~\ref {algo:reroot}, Section~\ref {sec:reroot})
	\STATE Construct a new pseudo-tree rooted at the new root (Online Appendix~2)
	\STATE Exchange codenames for $x$ and its domain $D_x$ (Algorithm~\ref {algo:P_DPOP}, lines \ref {algo:P_DPOP:codenames} to~\ref {algo:P_DPOP:send_codenames})\label {algo:P_DPOP_value:codenames}
	\STATE Choose and exchange obfuscation key for $x$ (Algorithm~\ref {algo:P_DPOP}, lines \ref {algo:P_DPOP:got_keys} to~\ref {algo:P_DPOP:send_keys})
	\STATE Propagate feasibility values up the pseudo-tree (Algorithm~\ref {algo:UTILpropagation}, except line~\ref {algo:UTILpropagation:argmin}) 	\label {algo:P_DPOP_value:UTILpropagation}
	\STATE \textbf {if} $x$ is root \textbf {then} Add local constraint $x = x^*$, with $x^*$ from Algorithm~\ref {algo:UTILpropagation}, line~\ref {algo:UTILpropagation:root_value} 	\label {algo:P_DPOP_value:ground}
\ENDWHILE
\end {algorithmic}
\caption {Overall P$^{\sfrac{3}{2}}$-DPOP$^+$ algorithm with full decision privacy, for variable~$x$}
\label {algo:P_DPOP_value}
\end {algorithm}

\subsection {Overview of the Algorithm}
\label {sec:P_DPOP_value_overview}

Algorithm~\ref {algo:P_DPOP_value} patches the decision privacy leak in P-DPOP$^+$ by removing its decision propagation phase. Only the root variable is assigned a value, and in order for all variables to be assigned values, each variable is made root in turn (unless the first feasibility propagation has revealed that the problem is infeasible, in which case the algorithm can terminate early). The intuition behind this P$^{\sfrac{3}{2}}$-DPOP$^+$ algorithm is therefore that P-DPOP$^+$'s bottom-up feasibility propagation phase is repeated multiple times, each time with a different variable~$x$ as the root of the pseudo-tree (lines \ref {algo:P_DPOP_value:while} to~\ref {algo:P_DPOP_value:UTILpropagation}). At the end of each iteration, a constraint $x = x^*$ is added to the problem to enforce consistency across iterations (line~\ref {algo:P_DPOP_value:ground}).

\subsection {Choosing a New Root Variable}
\label {sec:reroot}

To iteratively reroot the pseudo-tree, we propose to use an improved version of the rerooting procedure we initially introduced for the P$^2$-DPOP algorithm \cite {Leaute09a}. This procedure requires that each of the $n$ variables be assigned a unique ID; an algorithm to achieve this is presented in Online Appendix~3. This algorithm reveals to each variable~$x$ its unique ID~$id_x$, as well as a tight strict lower bound on the next unique ID~$id_x^+$ (i.e. the next unique ID equals $id_x^+ + 1$), and an upper bound $n^+$ on the total number of variables. Each variable~$x$ then creates a Boolean vector~$vector_x$ with a single zero entry at the index corresponding to its unique ID~$id_x$ (Algorithm~\ref {algo:P_DPOP_value}, line~\ref {algo:P_DPOP_value:vector}); this vector is then shuffled using a random permutation used to hide the sequence in which variables become roots. 

To keep the permutation secret, the vector is first encrypted using ElGamal encryption (Appendix~\ref {sec:ElGamal}), based on a compound public key jointly produced by the agents (Algorithm~\ref {algo:P_DPOP_value}, lines \ref {algo:P_DPOP_value:keys_start} to~\ref {algo:P_DPOP_value:keys_end}). This asymmetric encryption scheme enables each agent to (re-)encrypt the entries in the vectors using a common public key, such that the decryption can only be performed collaboratively by all the agents, using their respective private keys. 

\begin {algorithm}[htbp]
\begin {algorithmic}[1]
\ENSURE \textsc{ShuffleVectors}() for variable~$x$
\STATE $myID \leftarrow $ large random number 	\label {algo:reroot:myID}
\STATE $p_x \leftarrow$ random permutation of $[ 1 \ldots n^+ ]$

\vspace {5pt}
\STATE \COMMENT {Propagate $x$'s encrypted vector backwards along the circular ordering}
\STATE $vector_x \leftarrow E(vector_x)$ \COMMENT {encrypts the vector using the compound public key}
\STATE \textsc {ToPrevious}((VECT, $myID, vector_x$, 1)) as in Algorithm~\ref {algo:routing} in Appendix~\ref {sec:routing} 		\label {algo:reroot:start1}

\vspace {5pt}
\STATE \COMMENT {Process all received vectors}
\WHILE {\textbf {true}} 
	\STATE Wait for a message (VECT, $id, vector, round$) from the next variable
	
	\vspace {5pt}
	\IF {$round = 1$}
		\STATE \textbf {if} $id \neq myID$ \textbf {then} \textbf {for} $j = (id_x + 1) \ldots id_x^+$ \textbf {do} $vector[j] \leftarrow -1$ 		\label {algo:reroot:addMin1}
		\STATE \textbf {else} $round \leftarrow round + 1$ \COMMENT {$x$'s vector; move to next round} 							\label {algo:reroot:start2}
	\ENDIF
	
	\IF {$round > 1$ and $x$ is the current root} 																		\label {algo:reroot:endi}
		\STATE $round \leftarrow round + 1$ \COMMENT {the root starts each round except the first} 								\label {algo:reroot:starti}
	\ENDIF
	
	\STATE \textbf {if} $round = 3$ \textbf {then} $vector \leftarrow p_x(vector)$ \COMMENT {shuffle the vector} 						\label {algo:reroot:shuffle}
	\IF [done processing $vector_x$] {$round = 4$ and $id = myID$}
		\STATE $vector_x \leftarrow vector$ 																			\label {algo:reroot:end4}
		\STATE \textbf {continue}
	\ENDIF
	
	\vspace {5pt}
	\STATE \COMMENT {Pass on the vector backwards along the circular ordering}
	\STATE $vector \leftarrow E(vector)$ \COMMENT {re-encrypts the vector using the compound public key}
	\STATE \textsc {ToPrevious}((VECT, $id, vector, round$)) as in Algorithm~\ref {algo:routing} in Appendix~\ref {sec:routing}
\ENDWHILE

\vspace{3mm}

\ENSURE \textsc{Reroot}() for variable~$x$
\STATE \textbf {repeat} $entry \leftarrow $ \textsc {Decrypt}$(pop(vector_x))$ \textbf {while} $entry \neq -1$ \COMMENT {as in Algorithm~\ref {algo:collaborative_decryption}}
\STATE \textbf {if} $entry = 0$ \textbf {then} $x$ is the new root
\end {algorithmic}
\caption {Algorithm to choose a new root, for variable~$x$}
\label {algo:reroot}
\end {algorithm}

The agents then proceed as in Algorithm~\ref {algo:reroot}. Each variable~$x$ first starts the procedure \textsc {ShuffleVectors}(), which is run only once --- this is a performance improvement over our previous work \cite {Leaute09a}, where it was performed at each iteration. All the vectors are passed from variable to variable in a round-robin fashion, using a circular message routing algorithm presented in Appendix~\ref {sec:routing}. Each agent applies a secret permutation to each vector to shuffle it. \textsc {ShuffleVectors}() proceeds in four rounds. During round~1 (started on line~\ref {algo:reroot:start1}, Algorithm~\ref {algo:reroot}), each vector makes a full round along the circular ordering, during which each variable~$x$ overwrites some of the entries with~$-1$ (line~\ref {algo:reroot:addMin1}), at the same positions it did for its own $vector_x$ (Algorithm~\ref {algo:P_DPOP_value}, line~\ref {algo:P_DPOP_value:vector}). These $-1$ entries account for the IDs in $[ id_x + 1, id_x^+ ]$ that have not been assigned to any variable (Online Appendix~3). Once $x$ has received back its own $vector_x$, it enters the incomplete round~2 (line~\ref {algo:reroot:start2}) during which $vector_x$ is passed on until it reaches the current root (line~\ref {algo:reroot:endi}). The root then starts round~3 (line~\ref {algo:reroot:starti}), during which each variable~$x$ shuffles each vector using its secret permutation~$p_x$ (line~\ref {algo:reroot:shuffle}). The incomplete round~4 returns the fully shuffled vector to its owner (line~\ref {algo:reroot:end4}). 

To reroot the variable ordering at the beginning of each iteration of P$^{\sfrac {3}{2}}$-DPOP$^+$, each variable~$x$ calls the procedure \textsc {Reroot}(), which removes and decrypts the first element of $vector_x$. Entries that decrypt to~$-1$ correspond to unassigned IDs and are skipped. The single entry that decrypts to~0 identifies the new root. The decryption process (Algorithm~\ref {algo:collaborative_decryption}) is a collaborative effort that involves each variable using its private ElGamal key to partially decrypt the cyphertext, which travels around the circular variable ordering in the same way as the vectors, until it gets back to its sender variable, which can finally fully decrypt it.

\begin {algorithm}[htbp]
\begin{algorithmic}[1]
\ENSURE \textsc{Decrypt}$(e)$ for variable~$x$

\STATE $codename \leftarrow $ large random number used as a secret codename for~$x$
\STATE $codenames_x \leftarrow codenames_x \cup \{ codename \}$
\STATE \textsc{ToPrevious}((DECR, $codename$, $e$)) as in Algorithm~\ref {algo:routing}
\STATE Wait for message (DECR, $codename$, $e'$) from next variable in the ordering
\RETURN decryption of $e$ using $x$'s private key

\vspace{3mm}

\ENSURE \textsc{CollaborativeDecryption}$()$ for variable~$x$

\LOOP
	\STATE Wait for a message (DECR, $c$, $e$) from next variable in the ordering
	\IF {$c \not\in codenames_x$}
		\STATE $e' \leftarrow $ partial decryption of $e$ using $x$'s private key
		\STATE \textsc{ToPrevious}((DECR, $c$, $e'$)) as in Algorithm~\ref {algo:routing}
	\ENDIF
\ENDLOOP
\end{algorithmic}
\caption {Collaborative decryption of a multiply-encrypted cyphertext~$e$}
\label {algo:collaborative_decryption}
\end {algorithm}

\subsection {Algorithm Properties}
\label {sec:P_DPOP_value_properties}

We first analyze the completeness and complexity properties of the P$^{\sfrac{3}{2}}$-DPOP$^{(+)}$ algorithms, and then we move on to their privacy properties. 

\subsubsection {Completeness and Complexity}

\begin {theorem}
\label {thm:P_DPOP_value}
Provided that there are no codename clashes, the P$^{\sfrac{3}{2}}$-DPOP$^+$ algorithm terminates and returns a feasible solution to the DisCSP, if there exists one. 
\end {theorem}

\begin {proof}
On the basis of Theorem~\ref {thm:P_DPOP}, it remains to prove that the rerooting Algorithm~\ref {algo:reroot} terminates and is correct, and that the overall algorithm remains correct. The latter is easy to prove: at each iteration, a feasible value is found for the root variable (if there exists one), and that value is necessarily consistent with the chosen assignments to previous roots since these assignments are enforced by new, additional constraints (Algorithm~\ref {algo:P_DPOP_value}, line~\ref {algo:P_DPOP_value:ground}). 

When it comes to the rerooting procedure, the unique ID assignment algorithm (Online Appendix~3) ensures that each of the $n$ variables gets a unique ID in $0 \ldots (n^+ - 1)$. Therefore, each variable has a 0~entry at a unique position in its vector (Algorithm~\ref {algo:P_DPOP_value}, line~\ref {algo:P_DPOP_value:vector}). Round~1 of Algorithm~\ref {algo:reroot} also makes sure that all vectors have $-1$~entries at the same positions. This ensures that exactly one variable will become the new root at each iteration, since all vectors are applied the same sequence of permutations, and no variable will be root twice. 
\end {proof} 

In terms of complexity, P$^{\sfrac{3}{2}}$-DPOP$^+$ proceeds in a similar way to P-DPOP$^+$ (Section~\ref {sec:P_DPOP_properties}), except that the bottom-up feasibility propagation phase is repeated $n$~times (each time with a different root variable). The overall complexity in information exchange therefore becomes $O(n^2 \cdot D_{\max}^{sep_{\max}'})$, where $sep_{\max}'$ is the maximum separator size over all variables, \emph {and} over all iterations, which therefore is likely to be higher than the exponent for P-DPOP$^+$. The information exchanged by the rerooting protocol is negligible in comparison. The runtime complexity (measured in number of constraint checks) is also $O(n^2 \cdot D_{\max}^{sep_{\max}'})$, but the memory complexity is only $O(n \cdot D_{\max}^{sep_{\max}'})$, because removing the decision propagation phase makes it become unnecessary to compute and record $x^*(\tilde {p_x}, \cdot)$ (Algorithm~\ref {algo:UTILpropagation}, line~\ref {algo:UTILpropagation:argmin}). Our experimental results on graph coloring benchmarks (Section~\ref {sec:coloring}) suggest that the median value of $sep'_{\max}$ may only be greater than the median value of $sep_{\max}$ in P-DPOP$^+$ by a small multiplicative factor. In terms of number of ElGamal cryptographic operations, the rerooting procedure requires a total of $n(3n - 1)n^+ \in O(n^3)$ encryptions: each of the $n$~variables (re-)encrypts $(3n - 1)$ vectors of size~$n^+$ (each variable's vector performs 3 full rounds, except for the root's vector, which performs only 2 full rounds), with $n^+ \leq n + n \cdot 2incr_{\min}$, where $incr_{\min}$ is a constant input parameter of the algorithm. The procedure also requires a total of $n^2 n^+ \in O(n^3)$ collaborative decryptions: each of the $n$~variables (partially) decrypts $n$~vectors or size~$n^+$.

\subsubsection {Full Agent Privacy}

\begin {theorem}
The P$^{\sfrac{3}{2}}$-DPOP$^{(+)}$ algorithms guarantee full agent privacy. 
\end {theorem}

\begin {proof}
The unique ID assignment and circular routing algorithms guarantee full agent privacy, as demonstrated respectively in Online Appendix~3 and Appendix~\ref {sec:routing}. 

\begin{description}
\item [Pseudo-tree rerooting (Algorithm~\ref {algo:reroot})] The messages sent by \textsc {ShuffleVectors()} contain a variable ID, a vector of ElGamal cyphertexts, and a round number. The ID is used by the recipient to detect whether the vector is its own vector; it is a large random number chosen by the owner agent (Algorithm~\ref {algo:reroot}, line~\ref {algo:reroot:myID}), and therefore it cannot be linked to the identity of this owner agent by any other agent. The ElGamal vector and the round number also do not contain any information that could be used to identify the agent. Also note that the procedure used to exchange ElGamal public key shares (Algorithm~\ref {algo:P_DPOP_value}, lines \ref {algo:P_DPOP_value:keys_start} to~\ref {algo:P_DPOP_value:keys_end}) does not leak any information about agents' identities. The \textsc {Reroot()} procedure then makes use of the collaborative decryption algorithm, whose properties in terms of agent privacy are discussed below. 

\item [Collaborative decryption (Algorithm~\ref {algo:collaborative_decryption})] The procedure exchanges messages that contain an ElGamal cyphertext, and a codename used like the variable ID in Algorithm~\ref {algo:reroot}. This codename is similarly set to a large random number chosen by the current agent, and cannot be linked to the identity of this agent by any other agent. 
\end{description}This concludes the proof that the P$^{\sfrac{3}{2}}$-DPOP$^{(+)}$ algorithms guarantee agent privacy. 
\end {proof}

\subsubsection {Partial Topology Privacy}
\label {sec:topo:p1.5_dpop}

The topology privacy in P$^{\sfrac{3}{2}}$-DPOP$^{(+)}$ is only slightly worse than in P-DPOP$^{(+)}$. 

\begin {theorem}
The P$^{\sfrac{3}{2}}$-DPOP$^{(+)}$ algorithms guarantee \emph{partial} topology privacy. Each variable unavoidably discovers the total number of variables in the problem, and \emph {might} also discover a lower bound on a neighbor variable's degree in the constraint graph. The advantages of P$^{\sfrac{3}{2}}$-DPOP$^+$ over P$^{\sfrac{3}{2}}$-DPOP are the same as P-DPOP$^+$ over P-DPOP. 
\end {theorem}

\begin {proof}
Since there is one feasibility propagation phase per variable in the problem, the total number of variables inevitably becomes public. The following analyzes the topology privacy properties of each phase of P$^{\sfrac{3}{2}}$-DPOP$^{(+)}$ that is not already present in P-DPOP$^{(+)}$, except for unique ID assignment (Online Appendix~3) and secure message routing (Appendix~\ref {sec:routing}). 
\begin{description}
\item [Exchange of ElGamal key shares (Algorithm~\ref {algo:P_DPOP_value}, lines \ref {algo:P_DPOP_value:keys_start}--\ref {algo:P_DPOP_value:keys_end})] 
The messages containing ElGamal key shares do not contain any information that could be used to make inferences about the topology of the constraint graph. 

\item [Pseudo-tree rerooting (Algorithm~\ref {algo:reroot})] 
Each message travels along a circular variable ordering using the message routing algorithm in Appendix~\ref {sec:routing}, and contains: 
\begin{itemize}
\item a vector that is encrypted (and re-encrypted after each operation) and that therefore cannot provide any topological information; 
\item an $id$ that identifies the owner of the vector; being a secret, large random number, only the owner of the vector can identify itself; 
\item a round number can take the following values: 
	\begin{itemize}
	\item $round = 1$ only indicates that the vector is being modified, each variable setting in turn some of the values to~$-1$; 
	\item $round = 2$ only indicates that the vector is being sent to the root of the pseudo-tree. This does not happen for the vector of the (unknown) root; 
	\item $round = 3$ only indicates that the vector is being shuffled by each variable; 
	\item $round = 4$ only indicates that the vector is on its way back to its owner. This does not happen for the vector belonging to the (unknown) root. 
	\end{itemize}
\end{itemize}
\textsc {Reroot()} then uses the decryption algorithm whose properties are described below. 

\item [Collaborative decryption (Algorithm~\ref {algo:collaborative_decryption})] 
DECR messages are passed along the circular variable ordering, containing a secret codename for the original sender variable, which is the only variable capable of deciphering this codename. The last part of the message payload is an ElGamal cyphertext, which remains encrypted until it reaches back the original sender, and therefore does not leak any topological information. 

\end{description}
This concludes the proof that P$^{\sfrac{3}{2}}$-DPOP$^{(+)}$ guarantees \emph {partial} topology privacy. 
\end {proof}

\subsubsection {Partial Constraint Privacy}
\label {sec:constraint_privacy:obfuscation_rerooting}

The constraint privacy properties of the P$^{\sfrac{3}{2}}$-DPOP$^{(+)}$ algorithms differ from those of P-DPOP$^{(+)}$, because the former protect decision privacy (which benefits constraint privacy), but also reveal the total number of variables in the problem (which hurts constraint privacy). 

\begin {theorem}
The P$^{\sfrac{3}{2}}$-DPOP$^{(+)}$ algorithms guarantee \emph {partial} constraint privacy. The leaks are the same as in P-DPOP$^{(+)}$ (Section~\ref {sec:constraint_privacy:obfuscation}), but they happen less frequently. 
\end {theorem}

\begin{proof}
Single-variable feasibility messages leak the same amount of constraint privacy as in P-DPOP$^{(+)}$; notice however that, since the P$^{\sfrac{3}{2}}$-DPOP$^{(+)}$ algorithms now reveal the total number of variables, in some circumstances it may be possible for a variable to discover that a child is a leaf, and that the feasibility message it sends therefore contains information about its local subproblem only. However, multi-variable feasibility messages leak potentially much less information than in P-DPOP$^{(+)}$: consider again the simpler and non-restrictive case of a two-variable message $m(\tilde {p_x}, \tilde {y_i^x})$ received by~$\tilde {p_x}$. Because P$^{\sfrac{3}{2}}$-DPOP$^{(+)}$ now protects decision privacy, $\tilde {p_x}$ no longer discovers the value~$\tilde {y_i^x}^*$ chosen for~$\tilde {y_i^x}$, and is therefore no longer able to infer which of the entries corresponding to $\tilde {p_x} = \tilde {p_x}^*$ decrypts to~$0$. 

One exception is when the following three conditions simultaneously hold: 1)~P$^{\sfrac{3}{2}}$-DPOP is used, 2) the codename~$\tilde {y_i^x}$ refers to a variable~$y_i^x$ that is a neighbor of~$\tilde {p_x}$, and 3) $\tilde {y_i^x}^*$ is semi-private information to~$\tilde {p_x}$; then $\tilde {p_x}$ will still discover~$\tilde {y_i^x}^*$, and will be able to make the same inferences as in P-DPOP$^{(+)}$. If the first condition is not satisfied, i.e. P$^{\sfrac{3}{2}}$-DPOP$^+$ is used instead of P$^{\sfrac{3}{2}}$-DPOP, then $\tilde {p_x}$ will not be able to link the codename~$\tilde {y_i^x}$ to any known variable. This is also the case if P$^{\sfrac{3}{2}}$-DPOP is used, but the second condition does not hold. Finally, if the first two conditions hold, $\tilde {p_x}$ will only be able to discover~$\tilde {y_i^x}^*$ if it is semi-private information, i.e. if it can infer it only from its knowledge of the problem, and of its own chosen value~$\tilde {p_x}^*$. 
\end {proof}

\subsubsection {Full Decision Privacy}

\begin {theorem}
\label {thm:full_decision}
The P$^{\sfrac{3}{2}}$-DPOP$^{(+)}$ algorithms guarantee full decision privacy. 
\end {theorem}

\begin {proof}
The leak of decision privacy in P-DPOP$^{(+)}$ is fixed by removing the decision propagation phase. Instead, the variable ordering is rerooted, and the feasibility propagation phase is restarted. It is not possible to compare the feasibility messages received from one iteration to the next to infer the decision that has been made at the previous iteration: the messages are not comparable, since different codenames and obfuscation keys are used. 
\end {proof}

\section {P$^2$-DPOP$^+$: Adding Full Constraint Privacy}
\label {sec:ElGamal_UTILpropagation}

We now describe how the previous, non-fully secure obfuscation scheme can be replaced with ElGamal homomorphic encryption (Appendix~\ref {sec:ElGamal}) to achieve full constraint privacy, which corresponds to the original P$^2$-DPOP algorithm \cite {Leaute09a}, improved by the use of multiple codenames. An important limitation of the ElGamal scheme is that it is not \emph {fully homomorphic}: it is possible to compute the \emph {OR} of two encrypted Booleans, but it is only possible to compute the \emph {AND} of an encrypted Boolean with a \emph {cleartext} Boolean. As a consequence, the bottom-up feasibility propagation has to be performed on a variable ordering such that each variable can have only \emph {one} child, i.e. a \emph {linear} variable ordering (Figure~\ref {fig:circular_ordering}), using the message routing algorithm in Appendix~\ref {sec:routing}. Otherwise, in a pseudo-tree variable ordering, a variable with two children would not be able to join the two encrypted feasibility messages sent by the children. This could be addressed using the fully homomorphic encryption scheme by \citet {Gentry09}, however it is unclear whether this scheme would be practically applicable and would have sufficient performance. 

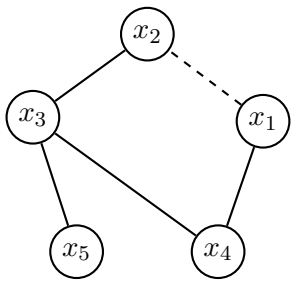
\begin{figure}[h]
\begin {center}
\begin {tikzpicture} [scale=.8]

\tikzstyle{var}=[circle, draw = black, thick, fill = white, minimum size = 0.7cm, inner sep=0pt, draw]

\node[var] (x1) at (16:2) {$x_1$};
\node[var] (x2) at (90:2) {$x_2$};
\draw (x1) -- (x2) [thick,dashed];

\node[var] (x3) at (162:2) {$x_3$};
\draw (x2) -- (x3) [thick];

\node[var] (x5) at (234:2) {$x_5$};
\draw (x3) -- (x5) [thick];

\node[var] (x4) at (306:2) {$x_4$};
\draw (x3) -- (x4) [thick];
\draw (x4) -- (x1) [thick];

\end {tikzpicture}

\caption{The (counter-clock-wise) circular variable ordering corresponding to Figure~\ref {fig:dfs}. }
\label{fig:circular_ordering}
\end {center}
\end{figure}

\begin{figure}[htbp]

\begin {flushleft}
~~~~~~~~~
\begin{tabular}{|c||c|c|c|}
\multicolumn {4}{c}{$x_3 \rightarrow x_2$}	\\
\hline
$x_2$	&	$R$				&	$B$				& 	$G$ 	\\
\hline \hline
		&	$\mathtt {true}$		&	$\mathtt {false}$	&	$\mathtt {true}$	\\
\hline
\end{tabular}
~~~~
\begin{tabular}{|c||c|c|c|}
\multicolumn {4}{c}{$x_4 \rightarrow x_5$}	\\
\hline
		&	\multicolumn{3}{c|}{$x_2$}	\\
$x_3$	&	$R$				&	$B$				& 	$G$ 	\\
\hline \hline
$R$		&	$\mathtt {true}$		&	$\mathtt {false}$	&	$\mathtt {true}$	\\
\hline
$B$		&	$\mathtt {true}$		&	$\mathtt {true}$		&	$\mathtt {true}$	\\
\hline
$G$		&	$\mathtt {true}$		&	$\mathtt {true}$		&	$\mathtt {true}$	\\
\hline
\end{tabular}
\end {flushleft}

\begin{center}
\begin {tikzpicture}

\tikzstyle{var}=[circle, draw = black, thick, fill = white, minimum size = 0.7cm, inner sep=0pt, draw]

\node[var] (x1) at (0, 0) {$x_1$};
\node[var] (x4) at (-2, 0) {$x_4$};
\draw[->] (x1) .. controls (-1, -.5) .. (x4) [thick];

\node[var] (x5) at (-4, 0) {$x_5$};
\draw[->] (x4) .. controls (-3, .5) .. (x5) [thick];

\node[var] (x3) at (-6, 0) {$x_3$};
\draw[->] (x5) .. controls (-5, -.5) .. (x3) [thick];

\node[var] (x2) at (-8, 0) {$x_2$};
\draw[->] (x3) .. controls (-7, .5) .. (x2) [thick];

\end {tikzpicture}

\end{center}

\begin {flushright}
\begin{tabular}{|c||c|c|c|}
\multicolumn {4}{c}{$x_5 \rightarrow x_3$}	\\
\hline
		&	\multicolumn{3}{c|}{$x_2$}	\\
$x_3$	&	$R$				&	$B$				& 	$G$ 	\\
\hline \hline
$R$		&	$\mathtt {true}$		&	$\mathtt {false}$	&	$\mathtt {true}$	\\
\hline
$B$		&	$\mathtt {true}$		&	$\mathtt {true}$		&	$\mathtt {true}$	\\
\hline
$G$		&	$\mathtt {false}$	&	$\mathtt {false}$	&	$\mathtt {false}$	\\
\hline
\end{tabular}
~~~~
\begin{tabular}{|c||c|c|c|}
\multicolumn {4}{c}{$x_1 \rightarrow x_4$}	\\
\hline
		&	\multicolumn{3}{c|}{$x_2$}	\\
$x_4$	&	$R$				&	$B$				& 	$G$ 	\\
\hline \hline
$R$		&	$\mathtt {true}$		&	$\mathtt {true}$		&	$\mathtt {true}$	\\
\hline
$B$		&	$\mathtt {true}$		&	$\mathtt {true}$		&	$\mathtt {false}$	\\
\hline
$G$		&	$\mathtt {true}$		&	$\mathtt {false}$	&	$\mathtt {true}$	\\
\hline
\end{tabular}
~~~~~
\end {flushright}

\begin{center}
\caption{Multiparty dynamic programming computation (in cleartext) of a feasible value for variable~$x_2$, using a linear variable ordering based on Figure~\ref{fig:circular_ordering}. }
\label{fig:dynamic_programming}
\end{center}
\end{figure}
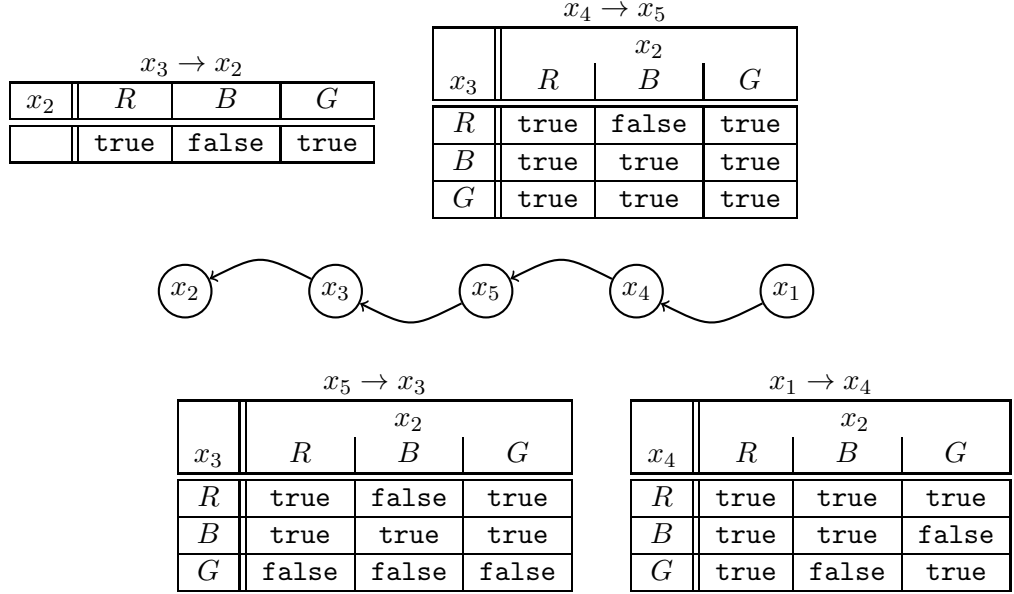

\subsection {Propagating Encrypted Feasibility Values along a Linear Variable Order}

In contrast to Figure~\ref {fig:dfs}, which illustrates multi-party dynamic programming on a pseudo-tree variable ordering (counting constraint violations), Figure~\ref {fig:dynamic_programming} shows (in cleartext) how it can be carried out on a linear ordering (in the Boolean domain). This assumes that a circular communication structure has preliminarily been set up as described in Appendix~\ref {sec:routing}. 

Algorithm~\ref {algo:onevar_propagate} gives the detailed pseudocode for this procedure, and is intended as a replacement for line~\ref {algo:P_DPOP_value:UTILpropagation} in Algorithm~\ref {algo:P_DPOP_value}. The differences with the pseudo-tree-based Algorithm~\ref {algo:UTILpropagation} are the following. First, while Algorithm~\ref {algo:UTILpropagation} initially reformulated the DisCSP into a Max-DisCSP so as to minimize the number of constraint violations, Algorithm~\ref {algo:onevar_propagate} works directly on the original DisCSP problem. This means that the conjunction operator~$\wedge$ replaces the sum operator (lines \ref {algo:nevar_propagate:local_join} and~\ref {algo:onevar_propagate:join}), and the disjunction operator~$\vee$ replaces the operator~$\min$ (line~\ref {algo:onevar_propagate:project}). Notice also that, in the case of the linear ordering, a variable's local subproblem no longer necessarily involves its parent variable in the ordering (line~\ref {algo:nevar_propagate:local_join}), just like $x_4$ shares no constraint with~$x_5$ in Figures \ref{fig:circular_ordering} and~\ref {fig:dynamic_programming}. 

\begin {algorithm}[htbp]
\begin{algorithmic}[1]

\STATE \COMMENT {Join local constraints:}
\STATE $m(x, \cdot) \leftarrow \bigwedge_{c \in \left\{ c' \in \mathcal {C} ~|~ x \in scope(c') ~\wedge~ scope(c') \cap \left( children_x \cup pseudo\_children_x \right) = \emptyset \right\}} c (x, \cdot)$ 	\label {algo:nevar_propagate:local_join}

\vspace{5pt}
\STATE \COMMENT {Apply codenames:}
\FOR {each $y_i \in \{ parent_x \} \cup pseudo\_parents_x$}
	\STATE $m( x, \cdot ) \leftarrow$ replace $(y_i, D_{y_i})$  in $m( x, \cdot )$ with $(\tilde {y_i^x}, \tilde {D_{y_i}^x})$ from Algorithm~\ref {algo:P_DPOP}, line~\ref {algo:P_DPOP:codenames}, and apply the permutation~$\sigma_{y_i}^x$ to~$\tilde {D_{y_i}^x}$ 	\label {algo:onevar_propagate:codenames}
\ENDFOR

\vspace{5pt}
\STATE \COMMENT {Join with received message:}
\STATE Wait for the message (FEAS, $m'( \cdot )$) from the next variable in the ordering
\FOR {each $z \in children_x \cup pseudo\_children_x$}
	\STATE $m'(\cdot) \leftarrow $ identify $(\tilde{x^{z}}, \tilde {D_x^{z}})$ as $(x, D_x)$ in $m'(\cdot)$ (if $\tilde{x^{z}}$ is present) 
\ENDFOR
\STATE $m( x, \cdot ) \leftarrow m(x, \cdot) \wedge m'( \cdot )$ \label {algo:onevar_propagate:join} 

\vspace{5pt}
\STATE \COMMENT {Project out $x$:}
\IF{$x$ is not the root variable}

	\STATE $m( \cdot ) \leftarrow E \left( \bigvee_x m( x, \cdot ) \right)$ \COMMENT {re-encrypts using the compound public key} \label {algo:onevar_propagate:project}

	\STATE \textsc{ToPrevious}((FEAS, $m( \cdot )$)) as in Algorithm~\ref {algo:routing} \label {algo:onevar_propagate:send}
\ENDIF

\STATE \textbf {else} $x^* \leftarrow$ \textsc{FeasibleValue}$( m( x, \cdot ))$ as in Algorithm~\ref {algo:log_decryption} 	\label {algo:onevar_propagate:decrypt}
\end{algorithmic}
\caption{Propagating feasibility values along a linear ordering, for variable~$x$}
\label{algo:onevar_propagate}
\end {algorithm}

The next  difference is that variable~$x$ no longer partially de-obfuscates its feasibility matrix before projecting itself (Algorithm~\ref {algo:UTILpropagation}, line~\ref {algo:UTILpropagation:deobfuscation}). The reason is that the ElGamal scheme is homomorphic, and therefore it is no longer necessary to first (partially) decrypt the feasibility values to project~$x$ using the operator~$\vee_x$. Only the root variable requires decryption (Algorithm~\ref {algo:onevar_propagate}, line~\ref {algo:onevar_propagate:decrypt}) to find a value~$x^*$ for its variable~$x$ whose encrypted feasibility value decrypts to~\texttt {true} (if any). This is described in the following section.

\subsection {Decrypting a Feasible Value for the Root Variable} 

The decryption of feasibility values at the root is a collaborative process in which each variable partially decrypts the cyphertext using its private key (Algorithm~\ref {algo:collaborative_decryption}). The dichotomy procedure in Algorithm~\ref {algo:log_decryption} uses at least $\lceil \log_2 |D_x| \rceil$ and at most $\lceil \log_2 {|D_x|} + 1 \rceil$ decryptions to find a feasible assignment to the root variable, or to detect infeasibility. 

\begin {algorithm}[h]
\begin{algorithmic}[1]
\ENSURE \textsc{FeasibleValue}$\left( m \left(x = x_{i_l} \ldots x_{i_r} \right) \right)$

\IF [cut in half the remaining subdomain:] {$i_l < i_r$} 
	\STATE $I \leftarrow \left[ i_l, \left\lfloor \frac{i_l+i_r}{2} \right\rfloor \right]$
	
	\STATE $feasible \leftarrow $ \textsc {Decrypt}$\left( \bigvee_{i \in I} m \left( x = x_i \right) \right)$ as in Algorithm~\ref {algo:collaborative_decryption}
	\STATE \textbf {if} $feasible = \texttt {true}$ \textbf {then} \textbf {return} \textsc{FeasibleValue}$\left( m \left( x = x_{i \in I} \right) \right)$
	\STATE \textbf {else return} \textsc{FeasibleValue}$\left( m \left( x = x_{i \in [i_l, i_r] - I} \right) \right)$
	
\vspace {5pt}
\ELSE [only one value remains for $x$] 
	\STATE $feasible \leftarrow $ \textsc {Decrypt}$\left( m \left( x = x_{i_l} \right) \right)$ as in Algorithm~\ref {algo:collaborative_decryption}
	\STATE { \textbf{if} $feasible = $ \TRUE \textbf { then return} $x_{i_l}$ \textbf {else return null}} 
\ENDIF
\end{algorithmic}
\caption{Finding a feasible value in the encrypted feasibility matrix~$m(x)$}
\label{algo:log_decryption}
\end {algorithm}

\subsection {Algorithm Properties}

We first analyze the completeness and complexity properties of the P$^2$-DPOP$^{(+)}$ algorithms, and then we move on to their privacy properties. 

\subsubsection {Completeness and Complexity}

\begin {theorem}
Provided that there are no codename clashes, the P$^2$-DPOP$^+$ algorithm terminates and returns a feasible solution to the DisCSP, if there exists one. 
\end {theorem}

\begin {proof}
Termination follows from Theorem~\ref {thm:P_DPOP_value}, and from the fact that the message routing procedure in Appendix~\ref {sec:routing} guarantees all feasibility messages eventually reach their destinations. When it comes to completeness, the homomorphic property of the ElGamal scheme ensures the projection of a variable~$x$ out of an encrypted feasibility matrix is correct, and that the feasibility message received by each variable in the linear ordering summarizes the (encrypted) feasibility of the lower agents' aggregated subproblems, as a function of higher variables. In particular, the feasibility message received by the root allows it to find a value for its variable that satisfies the overall problem, if there exists one. 
\end {proof}

The analysis of the complexity of the algorithm remains similar to the analysis in Section~\ref {sec:P_DPOP_value_properties}: it is $O(n^2 \cdot D_{\max}^{sep_{\max}''})$ in information exchange and in number of constraint checks, and $O(n \cdot D_{\max}^{sep_{\max}''})$ in memory, but $sep_{\max}''$ is now the maximum separator size \emph {along the successive linear variable orderings}, instead of along the pseudo-trees. The requirement that each variable may have at most one child tends to make this exponent increase significantly, as illustrated empirically in Section~\ref {sec:results}. In terms of number of ElGamal cryptographic operations, in addition to the cost of rerooting the variable ordering (Section~\ref {sec:P_DPOP_value_properties}), the algorithm also requires $O(n^2 \cdot D_{\max}^{sep_{\max}''})$ encryptions, and only $O(n \log D_{\max})$ collaborative decryptions.

\subsubsection {Full Agent Privacy}

\begin {theorem}
The P$^2$-DPOP$^{(+)}$ algorithms guarantee full agent privacy. 
\end {theorem}

\begin {proof}
The only changes introduced in P$^2$-DPOP$^{(+)}$ with respect to P$^{\sfrac{3}{2}}$-DPOP$^{(+)}$ are in feasibility propagation, and in finding a feasible value for the root variable. 

\begin{description}
\item [ElGamal feasibility propagation (Algorithm~\ref {algo:onevar_propagate})] From the point of view of agent privacy, this is the same procedure as Algorithm~\ref {algo:UTILpropagation}, but using Algorithm~\ref {algo:routing} for message routing, both of which algorithms guarantee agent privacy. 

\item [Root variable assignment (Algorithm~\ref {algo:log_decryption})] This consists in iteratively calling the procedure in Algorithm~\ref {algo:collaborative_decryption}, which has already been shown to guarantee agent privacy. 
\end{description}This concludes the proof that the P$^2$-DPOP$^{(+)}$ algorithms guarantee agent privacy. 
\end {proof}

\subsubsection {Partial Topology Privacy}

\begin {theorem}
The P$^2$-DPOP$^{(+)}$ algorithms guarantee \emph{partial} topology privacy. In addition to the limited leaks of topology privacy in P$^{\sfrac{3}{2}}$-DPOP$^{(+)}$, an agent \emph {might} also be able to discover that there exists another branch in the constraint graph that it is not involved in. 
\end {theorem}

\begin {proof}
There are only two relevant differences with P$^{\sfrac{3}{2}}$-DPOP$^{(+)}$: the linear variable ordering, and the choice of a value for the root variable that requires collaborative decryption. 

\begin{description}
\item [ElGamal feasibility propagation (Algorithm~\ref {algo:onevar_propagate})] 
To exchange FEAS messages along a linear variable ordering, the algorithm makes use of the circular message routing procedure, which is shown in Appendix~\ref {sec:routing} to guarantee full topology privacy. However, the last variable in the linear ordering needs to know it is the last in order to initiate the feasibility propagation; therefore, by contraposition, non-last variables know they are not the last, and, in particular, non-last leaves of the pseudo-tree discover the existence of another branch. This minor leak of topology privacy is already present in the unique variable ID assignment algorithm (Online Appendix~3). Besides this, the topology privacy properties of the feasibility propagation phases in P$^2$-DPOP and P$^2$-DPOP$^+$ are the same as in P-DPOP and P-DPOP$^+$, respectively. 

\item [Root variable assignment (Algorithm~\ref {algo:log_decryption})] 
This algorithm involves recursively calling the collaborative decryption procedure, shown to guarantee full topology privacy. 

\end{description}
This concludes the proof that P$^2$-DPOP$^{(+)}$ guarantees \emph {partial} topology privacy. 
\end {proof}

\subsubsection {Full Constraint Privacy}

\begin {theorem}
The P$^2$-DPOP$^{(+)}$ algorithms guarantee full constraint privacy. 
\end {theorem}

\begin {proof}
The P$^2$-DPOP$^{(+)}$ algorithms fix all the leaks of constraint privacy in P$^{<2}$-DPOP$^{(+)}$, by replacing the cryptographically insecure obfuscation through addition of random numbers, by the cryptographically secure ElGamal encryption (Appendix~\ref {sec:ElGamal}). This makes it no longer possible to compare two encrypted feasibility values without decrypting them, which would require the collaboration of all agents (or an amount of computation to break the encryption that can be made arbitrarily high in the worst case by increasing the ElGamal key size). In particular, while it is possible to compute the logical \emph {OR} of two cyphertexts without decrypting them, the result remains encrypted, and cannot be compared to the two inputs to decide which one is~\texttt {true}, if any. 
\end {proof}

\subsubsection {Full Decision Privacy}

\begin {theorem}
The P$^2$-DPOP$^{(+)}$ algorithms guarantee full decision privacy. 
\end {theorem}

\begin {proof}
The same proof applies as to Theorem~\ref {thm:full_decision}. 
\end {proof}

\section {Experimental Results}
\label {sec:results}

We report the empirical performance of our algorithms against the state-of-the-art MPC-DisCSP4 algorithm, on four classes of benchmarks: graph coloring, meeting scheduling, resource allocation, and game equilibrium. We only compare to MPC-DisCSP4, because to our knowledge it is the only other general DisCSP algorithm that provides strong privacy guarantees. For each problem class, the choice of the DisCSP formulation is crucial, because it dictates how the four types of privacy defined based on the DisCSP constraint graph will relate to the actual privacy of the original problem. In particular, the P$^*$-DPOP$^{(+)}$ algorithms use the standard DisCSP assumption that each constraint is known to all agents owning a variable in its scope (Section~\ref {sec:previous_privacy}). Therefore, when an agent wants to hide a constraint from neighboring agents, it must express its constraint over \emph {copies} of its neighbors' variables. Additional equality constraints must be introduced to make copy variables equal to their respective original variables. In contrast, MPC-DisCSP4 does not make use of this DisCSP assumption, and therefore it does not need the introduction of copy variables. 

Our first performance metric is \emph{simulated time}~\cite{Sultanik07a}, which is used, when all agents are simulated on a single machine, to estimate the time it would have taken to solve the problem if they had run in parallel on dedicated machines (ignoring communication delays). The two other metrics are the number of messages and the amount of information exchanged. For each metric, we report the median over at least 100 problem instances, with 95\% confidence intervals. For the obfuscation in P$^{<2}$-DPOP$^{(+)}$, we used random numbers of $B =$ 128 bits, while P$^2$-DPOP$^{(+)}$ used 512-bit ElGamal encryption. MPC-DisCSP4 also used 512 bits for its Paillier encryption. For the unique variable ID generation procedure in P$^{>1}$-DPOP$^{(+)}$, the parameter $incr_{\min}$ was set to~10. All algorithms were implemented inside the Java-based FRODO platform for DisCSP~\cite{Leaute09b}, coupled with the CSP solver JaCoP~\cite{JaCoP}. The experiments were run on a 2.2-GHz, dual-core computer, with Java 1.6 and a Java heap space of 2~GB. The timeout was set to 10~min (wall-clock time).

\subsection {Graph Coloring}
\label {sec:coloring}

We first report the performance of the algorithms on distributed, 3-color graph coloring problems. The graphs were randomly generated with varying numbers of nodes, and an edge density fixed to~$0.4$. Notice that, with a fixed number of colors and a fixed edge density, increasing the number of nodes increases the degree of the graph, and therefore reduces the number of feasible solutions; this explains the trends in some of the following graphs. The DisCSP formulation involves one decision variable per node, and assumes that each variable is controlled by a single-variable agent. Notice that inter-agent constraints are binary inequality constraints, and therefore decision privacy is relevant to this problem class: knowing one's chosen color is insufficient to infer the respective colors of one's neighbors. 

\begin{figure}[b!]
\begin{center}
\includegraphics [scale=1] {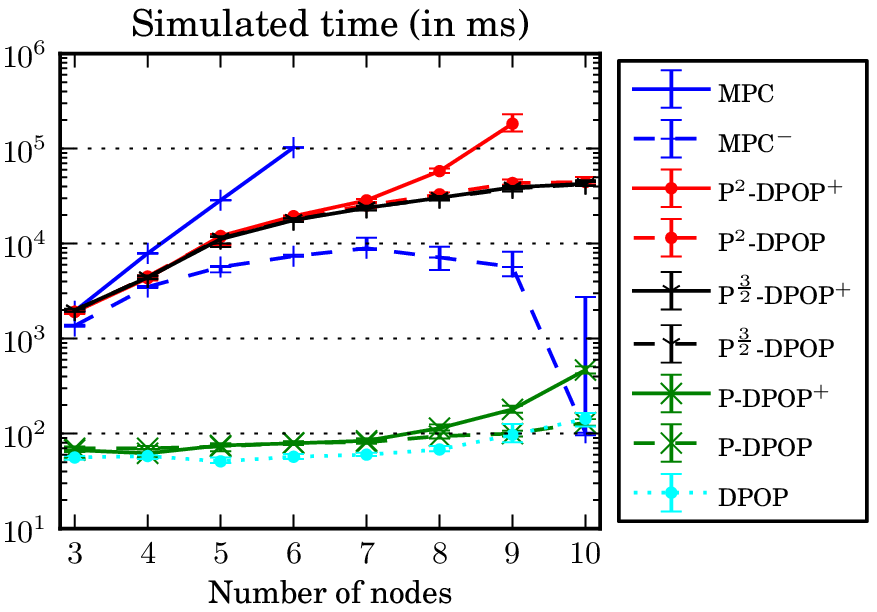}
\includegraphics [scale=1] {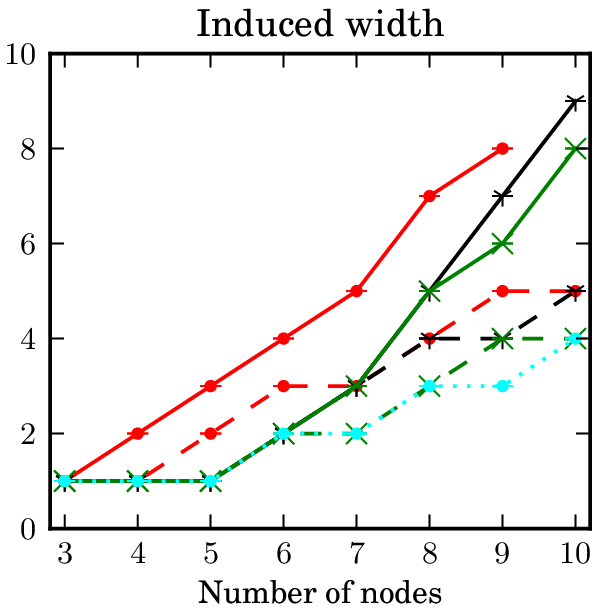}

\includegraphics [scale=1] {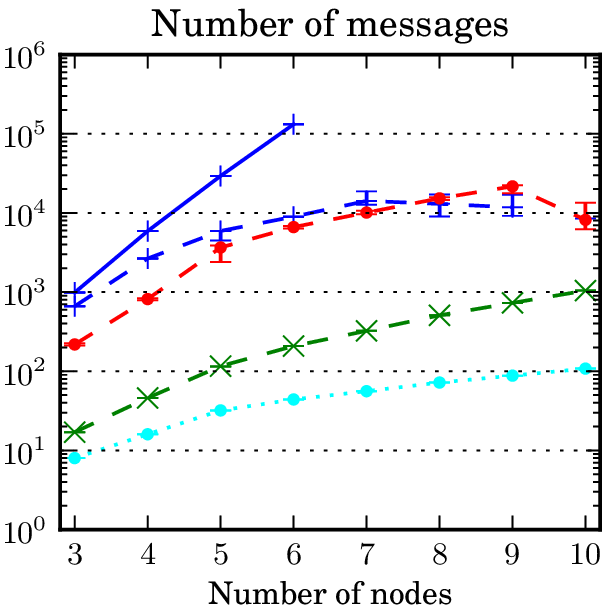}
\includegraphics [scale=1] {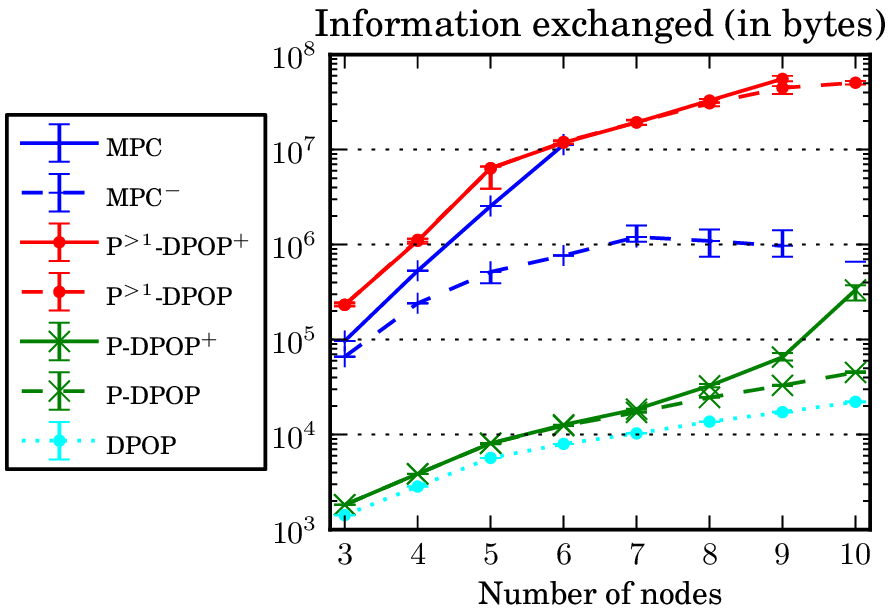}
\caption{Performance on graph coloring problems. }
\label{fig:graphColoring}
\end{center}
\end{figure}

To study the tradeoff between privacy and performance in MPC-DisCSP4, we considered a variant denoted MPC-DisCSP4$^-$, which assumes all inter-agent inequality constraints (i.e. node neighborhoods) are public, and only the final choice of colors is protected. Each agent first enumerates all feasible solutions to the overall problem (Section~\ref {sec:previous_privacy}), and then uses cryptographic techniques to securely and randomly choose one of the feasible solutions. If there exists none, the algorithm therefore terminates without any cryptographic operations nor exchanging messages. This explains the phase transition for MPC-DisCSP4$^-$ in the following graphs, since the probability of infeasibility increases with the problem size. 

Figure~\ref {fig:graphColoring} shows that MPC-DisCSP4 (denoted as \emph {MPC} in these and all subsequent figures) scales very poorly, timing out on problems with more than 6~nodes. MPC-DisCSP4$^-$ performs better; however, as mentioned before, it only protects the final choices of colors. For small numbers of nodes, the total state space is small, and MPC-DisCSP4$^-$ performs relatively well; for numbers of nodes above~9, the problem instances are mostly infeasible, and MPC-DisCSP4$^-$ quickly detects infeasibility without having to exchange any message. 

\begin{figure}[b!]
\begin{center}
\includegraphics [scale=1] {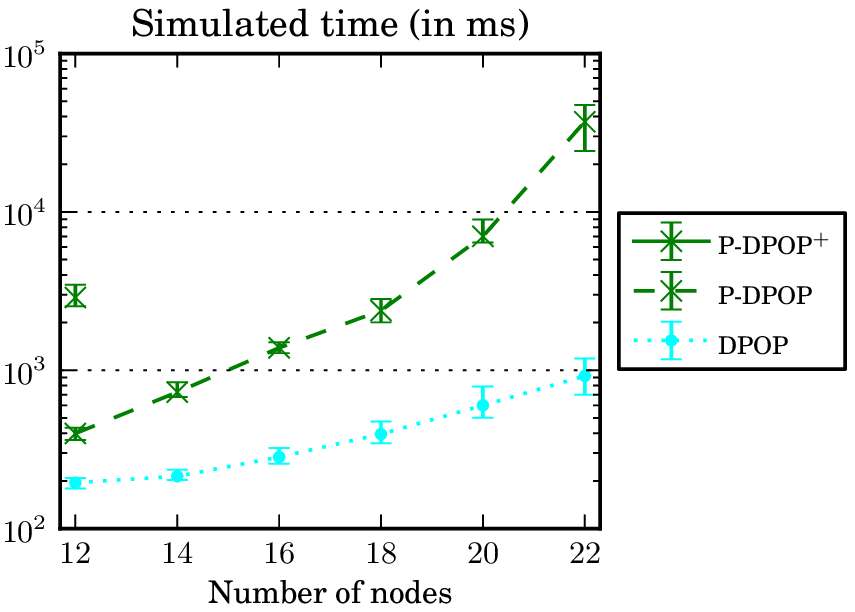}
\includegraphics [scale=1] {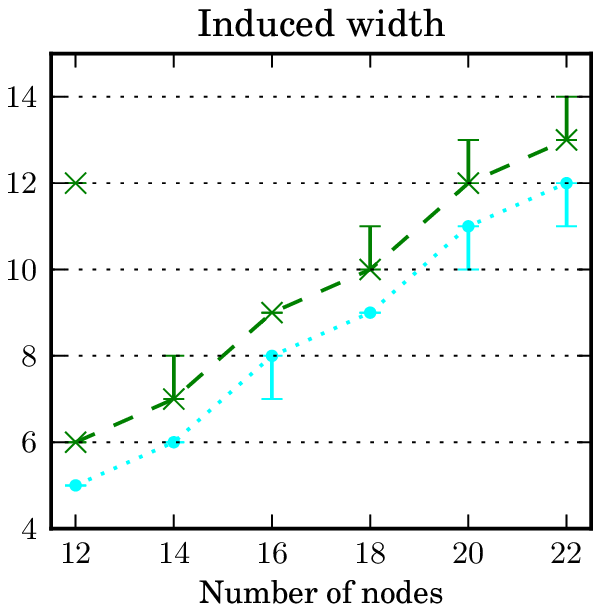}
\includegraphics [scale=1] {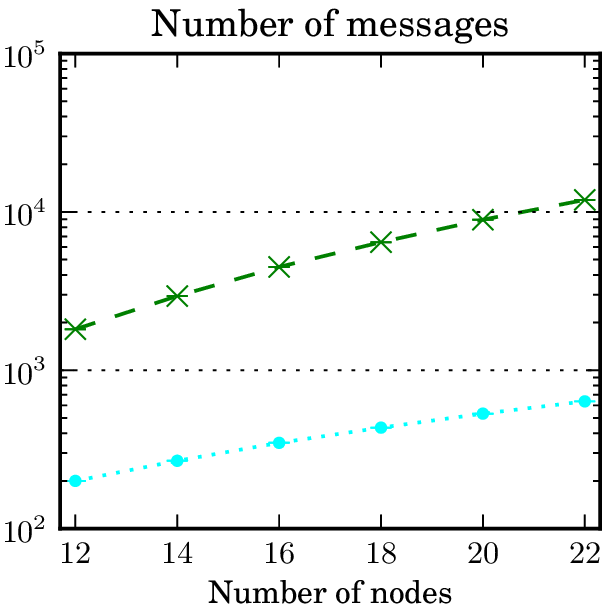}
\includegraphics [scale=1] {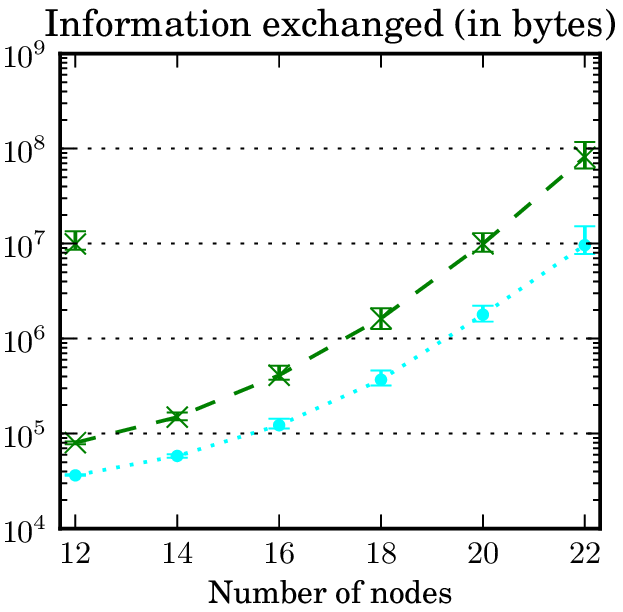}
\caption{Performance on larger graph coloring problems. }
\label{fig:graphColoring_large}
\end{center}
\end{figure}

The most efficient algorithms by far are P-DPOP$^{(+)}$, whose performance curves are at least one order of magnitude below all other algorithms. In particular, P-DPOP's runtime is sensibly the same as DPOP (the communication overhead is almost solely due to the root election algorithm). The cost of improved topology privacy in P-DPOP$^+$ vs. P-DPOP only starts to show for problem sizes above~7, when the induced widths of P-DPOP$^+$'s pseudo-trees start to deviate from P-DPOP and DPOP. Full decision privacy comes at much higher costs: P$^{\sfrac{3}{2}}$-DPOP$^{(+)}$'s curve is between 1 and 3 orders of magnitude above P-DPOP$^{(+)}$'s, even though their induced widths remain sensibly the same. This suggests that rerooting the pseudo-tree (which involves expensive cryptographic operations) is by far the complexity bottleneck, even when full constraint privacy is additionally guaranteed as in P$^2$-DPOP$^{(+)}$, whose linear variable orderings nevertheless have significantly higher induced widths than P$^{<2}$-DPOP$^{(+)}$'s pseudo-tree orderings. Notice that the slope of the runtime curve decreases as the problem size increases; this is due to the fact that more and more problems become infeasible, and the P$^{>1}$-DPOP$^{(+)}$ algorithms are able to terminate after the first iteration on infeasible problems. Similarly to P-DPOP$^+$ vs. P-DPOP, the cost of improved topology privacy is only visible above 7~nodes; P$^2$-DPOP$^+$ even timed out on problems of size~10. Finally, Figure~\ref {fig:graphColoring} illustrates the fact that MPC-DisCSP4 tends to send large numbers of small messages, while the P$^{>1}$-DPOP$^{(+)}$ algorithms send lower numbers of larger messages. 

Figure~\ref {fig:graphColoring_large} compares the performance of P-DPOP$^{(+)}$ against DPOP on larger graph coloring problem instances. On such larger problems, the improved topology privacy in P-DPOP$^+$ comes at a complexity price that is too high to scale above 12~nodes. On the other hand, P-DPOP's curves are only between one and two orders of magnitude above DPOP, and P-DPOP's median runtime on problem instances of size~22 is below 30~s.

\subsection {Meeting Scheduling}
\label {sec:meetings}

We now report experimental results on random meeting scheduling benchmarks. We varied the number of meetings, while keeping the number of participants per meeting to~2. For each meeting, participants were randomly drawn from a common pool of 3 agents. The goal is to assign a time to each meeting among 8~available time slots, such that no agent is required to attend simultaneous meetings. The pool of agents was deliberately chosen small to increase the complexity of the problems, by increasing the probability that each agent take part in multiple meetings. Note that fixing the pool size and the number of participants per meeting still generates an unbounded number of different problem instances as we increase the number of meetings, since the state space (the Cartesian product of the domains of the decision variables) keeps increasing with the number of meetings/decisions to be made. 

\begin{figure}[t!]
\begin{center}
\includegraphics [scale=1] {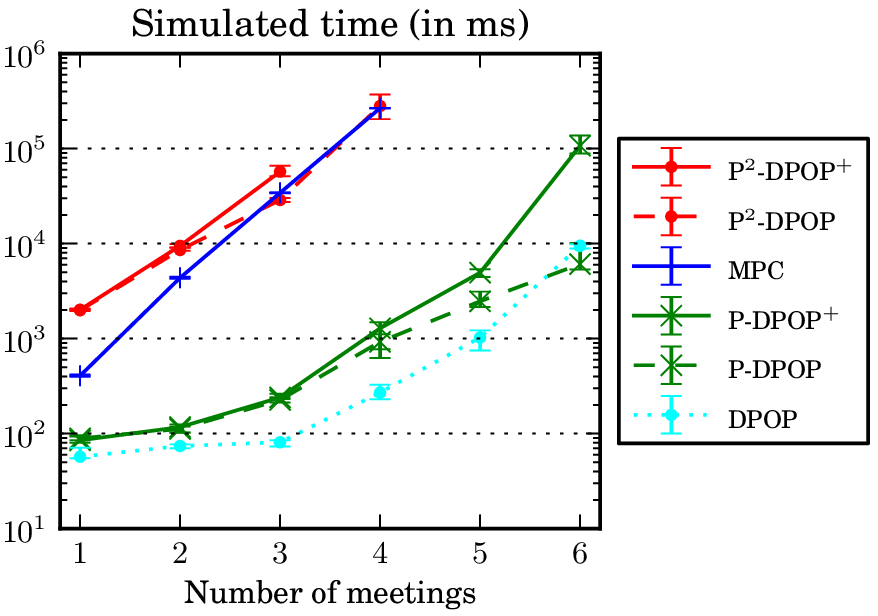}
\includegraphics [scale=1] {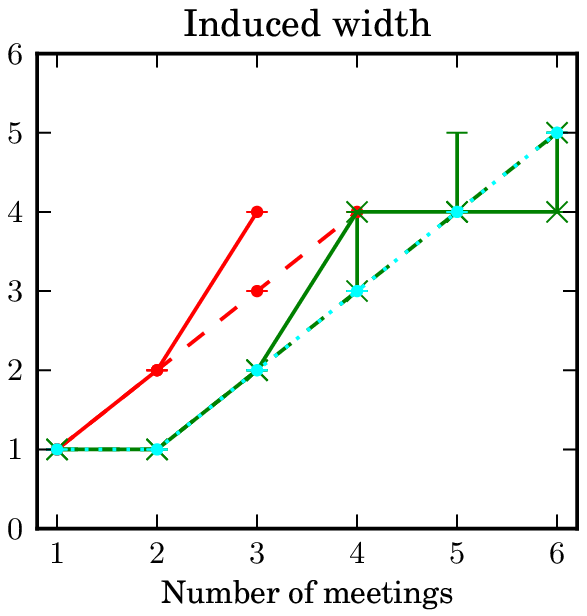}

\includegraphics [scale=1] {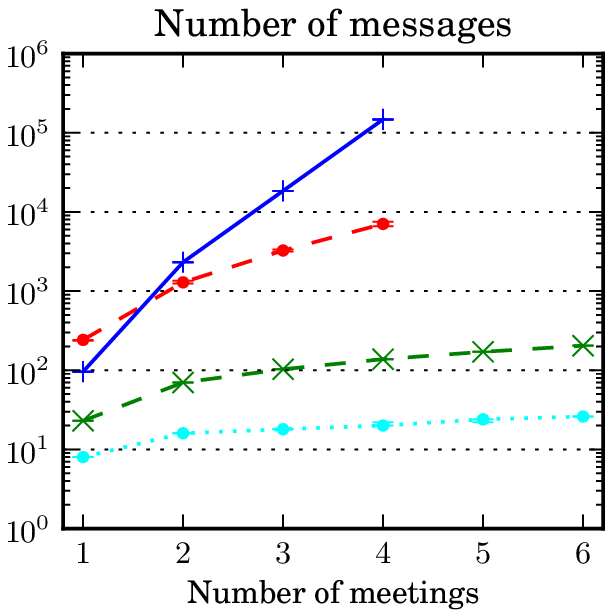}
\includegraphics [scale=1] {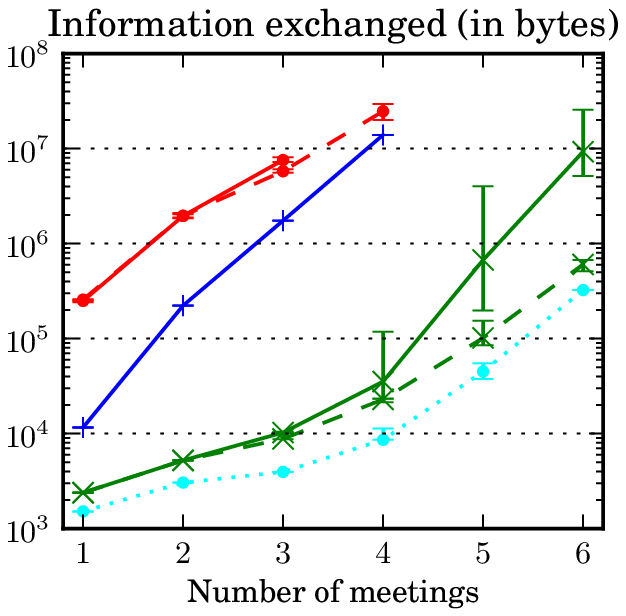}
\caption{Performance on meeting scheduling problems. }
\label{fig:meetings}
\end{center}
\end{figure}

The DisCSP formulation for this problem class was the following. Each agent owns one variable of domain size~8 for each meeting it participates in. There is an \emph {allDifferent} constraint over all its variables to enforce that all its meetings are scheduled at different times. For each meeting, a binary equality constraint expressed over the corresponding variables owned by the two participants enforces that the participants agree on the time for the meeting. Notice that all inter-agent constraints are binary equality constraints, and therefore P$^{\sfrac {3}{2}}$-DPOP$^{(+)}$ do not bring any additional privacy compared to P-DPOP$^{(+)}$, since the values of neighboring variables are semi-private information; therefore, we do not report the performance of P$^{\sfrac {3}{2}}$-DPOP$^{(+)}$. For MPC-DisCSP4, we simplified the formulation by only introducing one variable per meeting, owned by its initiator. This way, for each meeting, only its initiator is made public, but its exact list of participants remains secret (it is only revealed \emph{a posteriori} to the participants of the meeting when they attend it). 

As can be seen in Figure~\ref {fig:meetings}, P$^2$-DPOP's performance is comparable to that of MPC-DisCSP4 (but with much stronger privacy guarantees), although the former sends significantly more information on the smallest problems, but significantly fewer messages on the largest problems they could solve within the timeout limit. On the other hand, because it is a majority threshold scheme, MPC-DisCSP4 actually could not provide any privacy guarantees on problems of size~1, since they only involved 2~agents. Both algorithms could only scale up to problems of size~4, and timed out on larger problems. P$^2$-DPOP$^+$'s increased topology privacy comes at a price that made it time out earlier than P$^2$-DPOP; this complexity increase is due to P$^2$-DPOP$^+$'s steeper induced width curve. 

The P-DPOP$^{(+)}$ algorithms remain the most efficient by far: they perform between 1 and~2 orders of magnitude better than all others, both in terms of runtime and information exchanged. And like for graph coloring, the improved topology privacy in P-DPOP$^+$ comes at a price that is negligible for small problems, but can grow to one order of magnitude on problems of size~6, even if its induced width remains close to that of P-DPOP. In terms of runtime and information exchange, P-DPOP is only worse than DPOP by a small factor (since it has the same median induced width); however it sends approximately one order of magnitude more messages (which is mostly due to the pseudo-tree root election mechanism).

\subsection {Resource Allocation}
\label {sec:resource_allocation}

Next, we performed experiments on distributed resource allocation benchmarks. Problem instances were produced using the combinatorial auction problem generator CATS~\cite{Leyton-Brown00}, ignoring bid prices. We used the \emph {temporal matching} distribution modeling the allocation of airport takeoff/landing slots, fixing the total number of slots (i.e. resources) to~8, and varying the numbers of bids. Each bid is a request for a bundle of 2 resources (a takeoff slot and a corresponding landing slot). Multiple requests may be placed by the same airline company; each airline should have exactly one fulfilled.  

\begin{figure}[b!]
\begin{center}
\includegraphics [scale=1] {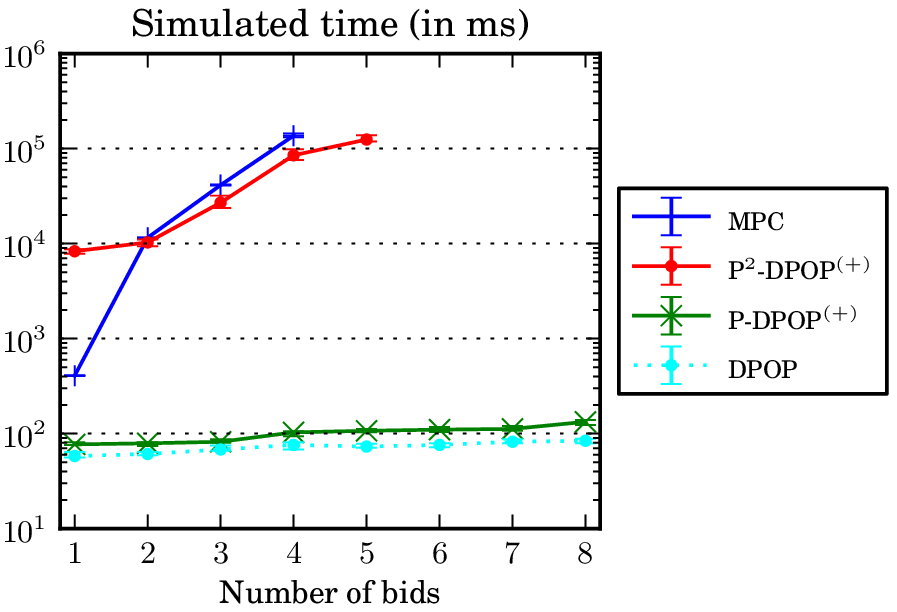}
\includegraphics [scale=1] {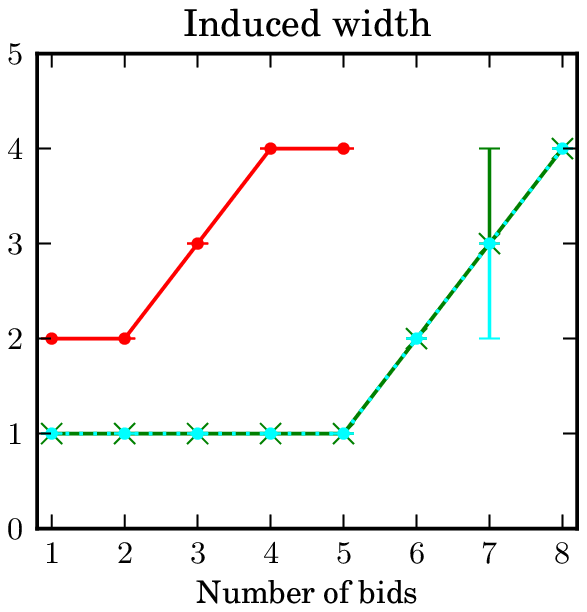}

\includegraphics [scale=1] {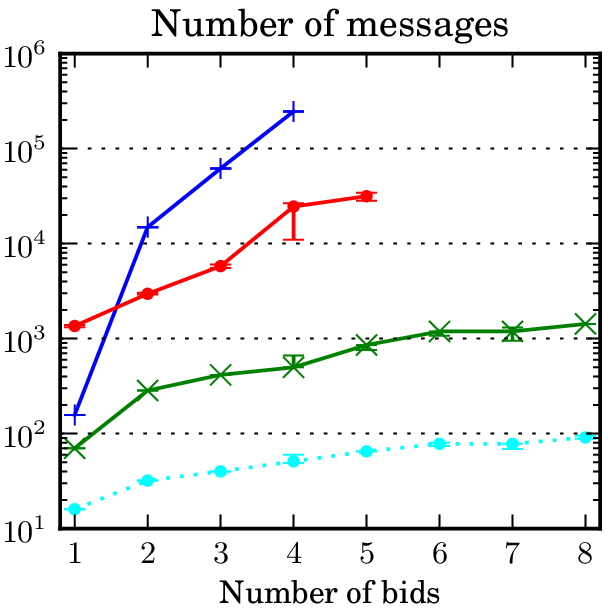}
\includegraphics [scale=1] {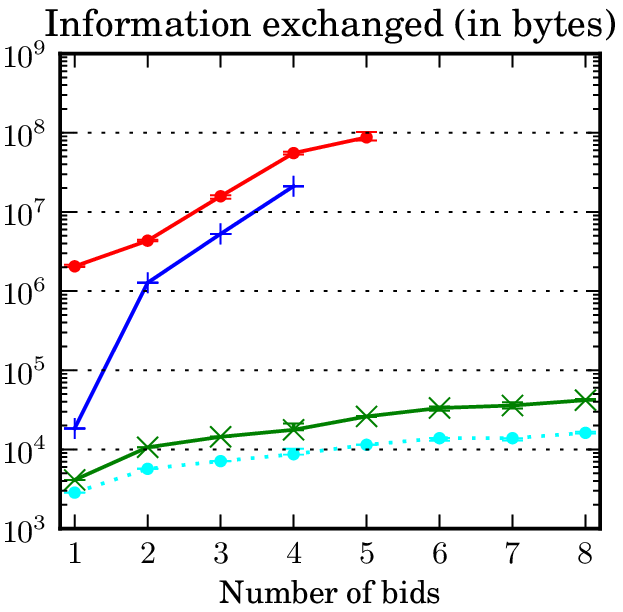}
\caption{Performance on resource allocation problems. }
\label{fig:auctions}
\end{center}
\end{figure}

This problem was modeled as a DisCSP as follows \cite {Leaute09a}. One agent is introduced for each bidder/airline and for each resource/slot, assuming that each resource is controller by a different resource provider/airport\footnote {CATS assumes there is a single auctioneer, and does not specify which slot is at which airport; this is why we have assumed that each resource was provided by a separate resource provider.}. For each resource~$X$, and for each bidder~$B$ that requests the resource, there is one binary variable~$x_b$ controlled by the resource provider, which models whether $B$ is allocated the resource ($x_b = 1$) or not ($x_b = 0$). The resource provider also expresses one constraint $\sum \leq 1$ over all her variables to enforce that her resource can only be allocated to at most one of the interested bidders. For each variable~$x_b$, we also introduce one copy variable~$b_x$ owned by the bidder~$B$, with the constraint~$x_b = b_x$. Each bidder~$B$ then expresses a constraint over all her variables, enforcing that she should only be allocated two resources that correspond exactly to one of her requests. The introduction of copy variables is motivated by the DisCSP assumption that each agent knows all constraints involving its variables, and serves two privacy-related purposes: 1) the full list of agents placing requests on a given resource is only known to the resource provider, and 2) the full list of resources requested by a given agent (and in which bundles) is only known to the agent itself. Like for the meeting scheduling problem class, all inter-agent constraints are equality constraints, therefore we do not report the performance of P$^{\sfrac {3}{2}}$-DPOP$^{(+)}$, whose privacy guarantees are the same as P-DPOP$^{(+)}$. 

For MPC-DisCSP4, the DisCSP formulation was simplified by not introducing copy variables hold by bidders, since they are not necessary to protect constraint privacy: bidders can request resources by expressing constraints directly over the variables owned by the resource providers. However, since MPC-DisCSP4 assumes that all variables are public, in order to increase topology privacy we introduced, for each resource, as many variables as bidders, regardless of whether they are actually interested in the resource. To reduce the size of the search space, we assumed that the $\Sigma \leq 1$ constraints were public. 

Figure~\ref {fig:auctions} shows that the performance of MPC-DisCSP4 decreases very fast with the number of requests, such that the algorithm was not able to scale beyond problems of size~4. The P$^2$-DPOP$^{(+)}$ algorithms seem to scale better, and were able to solve problems involving 5~requests. On all three metrics, both algorithms were largely outperformed by P-DPOP$^{(+)}$, whose runtime curve is remarkably flat, and almost overlaps with the runtime curve of DPOP, which is consistent with their undistinguishable induced width curves. The overhead of P-DPOP$^{(+)}$ compared to DPOP is slightly larger in terms of information exchanged, and goes up to one order of magnitude in terms of number of messages. P-DPOP$^+$ and P$^2$-DPOP$^+$ performed the same as their respective ``non-plus" variants.

\subsection {Strategic Game Equilibria}
\label {sec:party}

Finally, we report experimental results on one last class of problem benchmarks, which corresponds to the distributed computation of pure Nash equilibria in strategic games. We used the particular example of the \emph {party game} introduced by \citet {Singh04}, which is a one-shot, simultaneous-move, \emph {graphical game} \cite {Kearns01} in which all players are invited to a common party, and each player's possible strategies are whether to attend the party or not. Players are arranged in an undirected social graph, which defines which other invitees each player knows. Each player's reward for attending the party depends on whether her acquaintances also decide to attend, and on whether she likes them or not. The reward is 1 per attendee she likes, minus 1 per attendee she dislikes, and minus a constant cost of attendance in~$[0, 1]$. The reward for not attending is~0. 

\begin{figure}[b!]
\begin{center}
\includegraphics [scale=1] {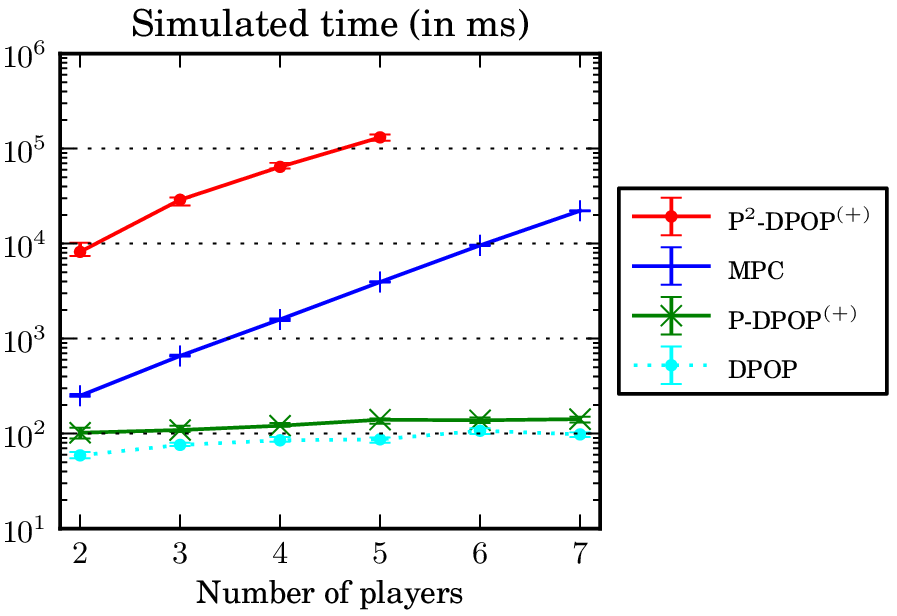}
\includegraphics [scale=1] {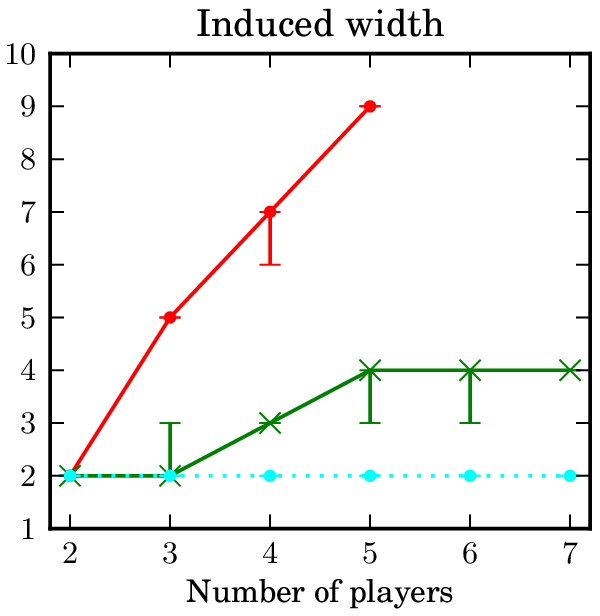}

\includegraphics [scale=1] {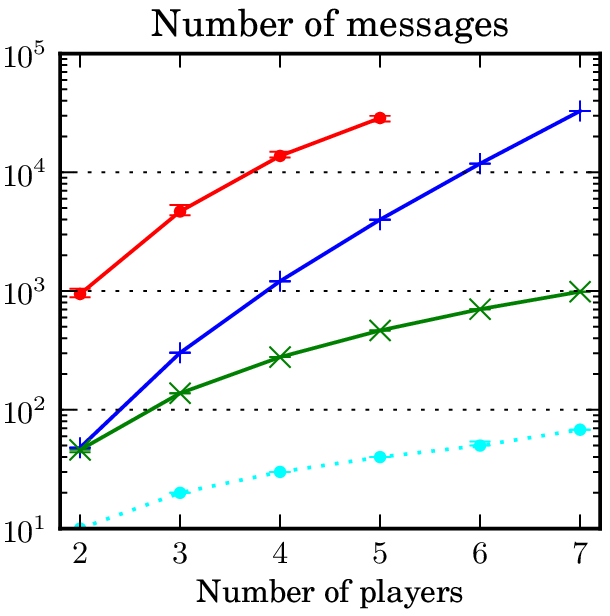}
\includegraphics [scale=1] {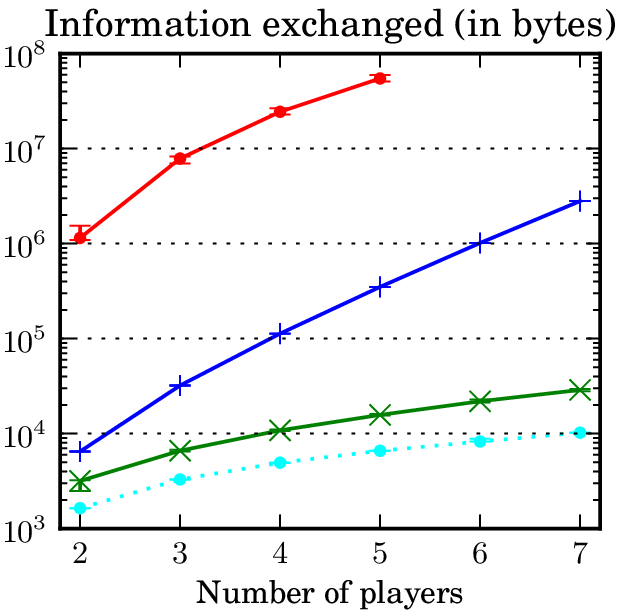}
\caption{Performance on party games. }
\label{fig:party}
\end{center}
\end{figure}

The problem of computing a Nash equilibrium to such a game can be formulated into a DisCSP as follows. Each player is an agent, which owns one binary variable for its strategy, and one copy variable for the strategy of each of its acquaintances. Each variable is constrained to be equal to each of its copy variables, using binary equality constraints like for resource allocation problems (Section~\ref {sec:resource_allocation}). Each agent also expresses one constraint over all its variables, which only allows a particular strategy for the agent if it is a best response to its neighbors' joint strategies. Notice that the resulting constraint graph is not the same as the game graph, due to the presence of copy variables. A solution to the DisCSP therefore yields a joint strategy profile for all players that is a pure Nash equilibrium, since each player plays best-response to her neighbors. Notice also that, since each player holds a copy variable for each of her neighbors' strategy, these strategies are semi-private information that cannot be protected, which is why we do not report the performance of P$^{\sfrac {3}{2}}$-DPOP$^{(+)}$. 

For MPC-DisCSP4, the DisCSP formulation can be simplified by not introducing copy variables \cite {Vickrey02}. An interesting consequence of this difference is that, contrary to P$^{\geq 1}$-DPOP$^{(+)}$, MPC-DisCSP4 is then able to hide each player's chosen strategy from her neighbors. In the context of the party game, this is not very useful to players who decide to attend the party, since they will necessarily eventually discover whether their acquaintances also decided to attend or not. On the other hand, a player who declines the invitation does not \emph {directly} discover anything about the list of attendees. She might still be able to make indirect inferences about the decisions of her acquaintances, based on the fact that her decision to decline is a best response to their respective chosen strategies. 

Figure~\ref {fig:party} reports on the performance of the algorithms on random acyclic game graphs of degree~2 (i.e. trees in which each node has at most 2~children), with varying numbers of players. The P$^2$-DPOP$^{(+)}$ algorithms were only able to scale up to problems of size~5 due to the rapidly increasing induced width, and were outperformed by MPC-DisCSP4 by at least one order of magnitude across all three metrics. Both algorithms still performed largely worse than the P-DPOP$^{(+)}$ algorithms, which are capable of scaling to much larger problems. This is because, in this setting, the induced width remains bounded: since the game graphs are acyclic, DPOP's induced width is constantly equal to~2, because each FEAS message sent by agent~$a_x$ to its parent agent~$a_y$ is expressed only over $a_y$'s strategy variable and the copy of $a_x$'s strategy variable held by~$a_y$. P-DPOP$^{(+)}$'s induced width is increased by 2 because agent~$a_y$ has at most 2 children in the pseudo-tree, each using a different codename for $a_y$'s strategy variable. As a result, the performance overhead in P-DPOP$^{(+)}$ compared to DPOP is minimal in terms of runtime; it is slightly larger in information exchanged, and reaches one order of magnitude in number of messages.

\section {Conclusion}

In this paper, we have addressed the issue of providing strong privacy guarantees in Distributed Constraint Satisfaction Problems (DisCSPs). We have defined four types of information about the problem that agents might want to hide from each other: \emph {agent privacy} (hiding an agent's identity from non-neighbors), \emph {topology privacy} (keeping the topology of the constraint graph private), \emph {constraint privacy} (protecting the knowledge of the constraints), and \emph {decision privacy} (the final value of each variable should only be known to its owner agent). Departing from previous work in the literature, which only addressed subsets of these privacy types, and often focused on quantifying the privacy loss in various algorithms, we have proposed a set of algorithms with strong guarantees about what information provably will not be leaked. 

We have carried out performance experiments on four different classes of benchmarks: graph coloring, meeting scheduling, resource allocation, and game equilibrium computation. The results show that our algorithms not only provide stronger privacy guarantees, but also scale better than the previous state of  the art. We have explored the tradeoff between privacy and performance: the P-DPOP$^+$ variant was shown to scale much better than the others, but can only guarantee partial constraint and decision privacy, which may still be considered sufficient in many problem classes. Full decision privacy (P$^{\sfrac {3}{2}}$-DPOP$^+$) and full constraint privacy (P$^2$-DPOP$^+$) come at significantly higher prices in computation time and information exchange, which, with today's hardware, limits their applicability to smaller problem instances. We have compared the performance of our algorithms against the MPC-DisCSP4 algorithm, which can be considered the previous state of the art in DisCSP with strong privacy guarantees. On the first three classes of benchmarks, all our algorithms almost systematically outperformed MPC-DisCSP4 in terms of runtime and number of messages exchanged; however, MPC-DisCSP4 proved to exchange less information than P$^{>1}$-DPOP$^+$. On game equilibrium computation, MPC-DisCSP4 scaled much better than P$^2$-DPOP$^+$ along all three metrics, but was still largely outperformed by P-DPOP$^+$. In terms of practical applicability, we have shown that some of our algorithms scale to medium-size problems that are beyond reach of the previous state of the art in general DisCSP with strong privacy guarantees. We have also investigated the application of these algorithms to real-life meeting scheduling, in collaboration with the Nokia Research Center in Lausanne.  

Future work could extend the techniques in this paper along several directions. First, while we have restricted ourselves to pure satisfaction problems for the sake of simplicity, our algorithms can be easily extended to solve Distributed Constraint Optimization Problems (DCOPs). In fact, our P$^{<2}$-DPOP$^+$ algorithms already are optimization algorithms; only P$^2$-DPOP$^+$ requires some changes to be applied to DCOPs. These changes involve replacing ElGamal-encrypted Boolean feasibility values with ElGamal-encrypted, bit-wise vector representations of integer cost values, as described by \citet {Yokoo02}. This would incur an increase in complexity that is only linear in an upper bound on the cost of the optimal solution. An optimization variant of MPC-DisCSP4, called \emph {MPC-DisWCSP4}, was also already proposed by \citet {Silaghi04}; we report performance comparisons with our algorithms in other publications \cite {Leaute11,Leaute11b}. 

Further avenues of future research could result from relaxing our assumption that agents are \emph {honest, but curious}. A number of challenging issues arise when attempting to apply the techniques in this paper to \emph {self-interested} agents that can manipulate the protocol in order to achieve solutions that better suit their selfish preferences. One such issue is that of \emph {verifiability}, which involves making it possible to check whether the protocols were executed as designed, without the need to decrypt the messages exchanged. Another interesting issue is whether it is possible to modify the algorithms to make them \emph {incentive-compatible}, such that it is in each agent's best interest to honestly follow the protocol. 


\appendix

\section {Cooperative ElGamal Homomorphic Encryption}
\label {sec:ElGamal}

\emph{Homomorphic encryption} is a crucial building block of the privacy-preserving algorithms introduced in this paper. \emph{Encryption} is the process by which a message --- in this appendix, a Boolean --- can be turned into a \emph{cyphertext}, in such a way that decrypting the cyphertext to retrieve the initial cleartext message is impossible (or, in this case, computationally very hard in the worst case) without the knowledge of the secret encryption key that was used to produce the cyphertext. An encryption scheme is said to be \emph {homomorphic} if it is possible to perform operations on cyphertexts that translate to operations on the initial cleartext messages, without the need to know the encryption key. \emph {ElGamal encryption} \cite{Elgamal85} is one such encryption scheme that possesses this homomorphic property.

\subsection {Basic ElGamal Encryption of Booleans}

ElGamal encryption can be used to encrypt Booleans such that performing the following operations on encrypted Booleans is possible without the knowledge of the decryption key: 
\begin{itemize}
\item the AND of an encrypted and a cleartext Boolean; 
\item the OR of two encrypted Booleans. 
\end{itemize}

ElGamal encryption is a homomorphic, public key cryptography system based on the intractability of the Diffie-Hellman problem \cite {Tsiounis98}, which proceeds as follows. Let $p$ be a \emph {safe prime} of the form $2rt+1$, where $r$~is a large random number, and $t$~is a large prime. All numbers and all computations will be modulo~$p$. Let~$g$ be a \emph {generator} of~$\mathbb {Z}_p^*$, i.e. $g$ is such that its powers cover~$[1, p-1]$. With $p$ and~$g$ assumed public knowledge, the ElGamal private key is a chosen random number $x \in [1, p-2]$, and the associated public key is $y=g^x$. A cleartext number~$m$ is then encrypted as follows: 
\begin{equation}
\label {eq:encrypt}
E(m) = (\alpha,\beta) = (my^r,g^r) 
\end{equation}
where $r$ is a random number chosen by the encryptor. Decryption proceeds as follows: 
$$\frac{\alpha}{\beta^x} = \frac{my^r}{(g^r)^x} = m~.$$

A useful feature of ElGamal encryption is that it allows to {\em randomize} an encrypted value to generate a new encryption bearing no similarity with the original value. Randomizing $E(m)$ in Eq.~(\ref {eq:encrypt}) yields: $$E^2(m) = (\alpha y^{r'}, \beta g^{r'}) = (m y^{r+r'}, g^{r+r'})$$ which still decodes to~$m$. To encrypt Booleans, we represent $\mathtt{false}$ by $1$, and $\mathtt{true}$ by a value $z \neq 1$, which allows us to compute the AND and OR operations: 
$$ E(m) \wedge \mathtt{false} = E(1) ~~~~~~~ E(m) \wedge \mathtt{true} 	= E^2(m)~;$$
$$E(m_1) \vee E(m_2)  = (\alpha_1 \cdot \alpha_2, \beta_1 \cdot \beta_2) = E(m_1\cdot m_2)~.$$

\subsection {Cooperative ElGamal Encryption}

In the previous ElGamal encryption scheme, decryption can be performed in a single step, using the private key, which is a secret of the agent that originally encrypted the message. However, it is also possible to perform ElGamal encryption in such a way that \emph {all} agents need to cooperate in order to perform decryption. This is possible through the use of a \emph {compound} ElGamal key $(x, y)$ that is generated cooperatively by all agents \cite{Pedersen91}:  
\begin{description}
\item [Distributed Key Generation] The ElGamal key pairs $(x_i, y_i)$ of $n$~agents can be combined in the following fashion to obtain the compound key pair~$(x, y)$: 
$$x = \Sigma_{i=1}^n x_i ~~~~~~~ y = \Pi_{i=1}^n y_i~.$$

\item [Distributed Decryption] If each agent publishes its decryption share $\beta^{x_i}$, the message can be decrypted as follows: 
$$\frac{\alpha}{\Pi_{i=1}^n \beta^{x_i}} = \frac{\alpha}{\beta^x} = m~.$$
\end{description}

\section {Routing of Messages along a Circular Variable Ordering}
\label {sec:routing}

In order to implement the round-robin exchange of vectors briefly presented in Section~\ref {sec:P_DPOP_value_overview}, the variables are ordered along a circular ordering that is mapped to the chosen pseudo-tree, as illustrated in Figure~\ref {fig:circular_ordering} (page~\pageref {fig:circular_ordering}) . Each variable needs to be able to send a message to the previous variable (i.e. clock-wise) in the ordering, which is a challenge in itself because only neighboring variables should communicate directly. Furthermore, to protect agent and topology privacy, no agent should know the overall circular ordering. To solve this issue, Algorithm~\ref {algo:routing} is the algorithm used in P$^2$-DPOP~\cite {Leaute09a} to route messages. 

\begin {algorithm}[h]
\begin{algorithmic}[1]
\ENSURE \textsc{ToPrevious}$(M)$ for variable~$x$
\STATE \textbf {if} $x$ is the root of the pseudo-tree \textbf {then} Send the message (LAST, $M$) to $x$'s last child 	\label {algo:routing:root}
\STATE \textbf {else} Send the message (PREV, $M$) to $x$'s parent 	\label {algo:routing:PREV_to_parent}

\vspace{3mm}

\ENSURE \textsc{RouteMessages}$()$ for variable~$x$
\LOOP
	\STATE Wait for an incoming message ($type, M$) from a neighbor~$y_i$
	
	\IF {$type = $ LAST}
		\STATE \textbf {if} $x$ is a leaf \textbf {then} Deliver message $M$ to $x$ 		\label {algo:routing:LAST_deliver}
		\STATE \textbf {else} Send the message (LAST, $M$) to $x$'s last child 		\label {algo:routing:to_last}
		
	\ELSIF {$type = $ PREV}
		\STATE \textbf {if} $y_i$ is $x$'s first child \textbf {then} Deliver the message $M$ to $x$ 		\label {algo:routing:PREV_deliver}
		\STATE \textbf {else} Send the message (LAST, $M$) to the child before $y_i$ in $x$'s list of children 		\label {algo:routing:first_LAST}
	\ENDIF
\ENDLOOP
\end{algorithmic}
\caption {Sending a message~$M$ clock-wise in the circular variable ordering. }
\label {algo:routing}
\end {algorithm}

Consider for instance a message~$M$ that agent~$a(x_1)$ wants to send to the previous variable --- which is~$x_4$, but $a(x_1)$ does not know it. Agent~$a(x_1)$ wraps $M$ into a PREV message that it sends to its parent variable~$x_4$ (line~\ref {algo:routing:PREV_to_parent}). Because the sender variable~$x_1$ is $x_4$'s first (and only) child, $a(x_4)$ infers that it should deliver~$M$ to itself (line~\ref {algo:routing:PREV_deliver}). Consider that $a(x_4)$ now wants to forward~$M$ to its previous variable --- $x_5$, which $a(x_4)$ does not know. Like before, $a(x_4)$ sends a message (PREV, $M$) to its parent variable~$x_3$, which then reacts by sending a message (LAST, $M$) to its last child preceding~$x_4$ in its list of children, which is~$x_5$ (line~\ref {algo:routing:first_LAST}). LAST messages indicate that the payload~$M$ should be delivered to the last leaf of the current subtree (line~\ref {algo:routing:to_last}); therefore, $a(x_5)$ delivers~$M$ to itself (line~\ref {algo:routing:LAST_deliver}) since it has no children. If the root wants to send a message to its previous variable, it also uses a LAST message to forward it to the last leaf of the overall pseudo-tree (line~\ref {algo:routing:root}). 

\begin {theorem}
Algorithm~\ref {algo:routing} guarantees full agent privacy.
\end {theorem}

\begin {proof}
The goal of this algorithm is precisely to address agent privacy issues in the pseudo-tree rerooting procedure, which involves each variable sending a message to the previous variable in a circular ordering of the variables. There is no guarantee that there exist a circular ordering such that any two consecutive variables are owned by neighboring agents, which is necessary to protect agent privacy. Therefore, Algorithm~\ref {algo:routing} is responsible for routing these messages through paths that only involve communication between neighboring agents. 

The routing procedure itself only involves encapsulating the routed messages inside PREV or LAST messages, which do not contain any other payload. Therefore, as long as the routed messages do not contain information that can be used to identify a non-neighboring agent, the routing procedure guarantees agent privacy. 
\end {proof}

\begin {theorem}
Algorithm~\ref {algo:routing} guarantees full topology privacy.
\end {theorem}

\begin {proof}
The purpose of this algorithm is to enable variables to propagate messages along a circular variable ordering, without the need to know any topological information about the constraint graph, other than the knowledge of their respective (pseudo-)parents and (pseudo-)children in the pseudo-tree. \textsc{ToPrevious()} makes it possible to send a message to the previous variable in the circular ordering, without knowing which variable this is. 
\begin{itemize}
\item The reception of a (PREV, $M$) message only indicates that the sender child wants the included message~$M$ to be delivered to its previous variable, which is either the recipient of the PREV message, or an unknown descendant thereof. 
\item The reception of a (LAST, $M$) message from one's parent indicates that an unknown variable (either the unknown root of the pseudo-tree, or the unknown child of an unknown ancestor, in another branch) wants $M$ to be delivered to its previous variable, which is one's descendant in the pseudo-tree. 
\end{itemize}
\end {proof}


\end{document}